\documentclass{article}
\usepackage[preprint,main]{neurips_2026}

\usepackage[utf8]{inputenc}
\usepackage[T1]{fontenc}
\usepackage{hyperref}
\usepackage{url}
\usepackage{amsmath,amssymb,amsfonts,amsthm}
\usepackage{xcolor}
\usepackage{booktabs}
\usepackage{algorithm}
\usepackage{algorithmic}
\usepackage{graphicx}
\usepackage{multirow}
\usepackage{enumitem}
\usepackage{microtype}
\usepackage{thmtools,thm-restate}
\newtheorem{assumption}{Assumption}
\newtheorem{statement}{Statement}
\newtheorem{lemma}{Lemma}

\title{Primal-Dual Guided Decoding for Constrained Discrete Diffusion}

\author{
  Federico Tomasi\thanks{These authors contributed equally to this work.} \\
  Spotify \\
  \texttt{federicot@spotify.com} \\
  \And
  Dmitrii Moor\footnotemark[1] \\
  Spotify \\
  \texttt{dmitriim@spotify.com} \\
  \And
  Alice Wang \\
  Spotify \\
  \texttt{alicew@spotify.com} \\
  \And
  Mounia Lalmas \\
  Spotify \\
  \texttt{mounia@acm.org} \\
}

\begin{document}

\maketitle

\begin{abstract}
Discrete diffusion models generate structured sequences by progressively unmasking tokens, but enforcing global property constraints during generation remains an open challenge. We propose primal-dual guided decoding, an inference-time method that formulates constrained generation as a KL-regularised optimisation problem and solves it online via adaptive Lagrangian multipliers. At each denoising step, the method modifies token logits through an additive, constraint-dependent bias, with multipliers updated by mirror descent based on constraint violation. The bias arises as the optimal KL-regularised projection of the constraint, so the constrained distribution remains as close as possible to the model's unconstrained distribution while still satisfying the constraint. The method requires no retraining and no additional model evaluations beyond standard sampling, supports multiple simultaneous constraints, and provides formal bounds on constraint violation. We evaluate our approach on topical text generation, molecular design, and music playlist generation, showing that a single algorithm instantiated via domain-specific scoring functions improves constraint satisfaction while preserving relevant domain-specific quality metrics.
\end{abstract}

\section{Introduction}
 
Discrete diffusion models~\citep{sahoo2024mdlm,austin2021d3pm} have recently emerged as a powerful alternative to autoregressive (AR) models for generating structured discrete sequences. MDLM~\citep{sahoo2024mdlm} progressively unmasks tokens from a fully masked state, enabling bidirectional context during generation. The paradigm applies to domains where outputs are discrete token sequences: natural language~\citep{sahoo2024mdlm,lou2024sedd}, molecular strings~\citep{irwin2020zinc20}, music playlists~\citep{rajput2023recommender}, and protein sequences~\citep{alamdari2023evodiff,wang2024dplm}.

However, practical applications require generated outputs to satisfy property constraints (e.g., topical vocabulary in text, target molecular weight in molecule generation, or target genre composition in music playlists) while remaining close to the learned data distribution. This creates a fundamental tension between constraint satisfaction and distributional fidelity. 
Existing approaches address this trade-off only partially: predictor- or classifier-based guidance~\citep{nisonoff2025unlocking,schiff2025simple} requires training an auxiliary model per property, and constrained discrete diffusion~\citep{cardei2025cdd} relies on inner optimisation loops with $5$--$20\times$ computational overhead.

We propose \textbf{primal-dual guided decoding} for discrete diffusion, an inference-time method that formulates constrained generation as a KL-regularised optimisation problem and solves it online during denoising.
The method is based on the stochastic primal-dual approximation technique, where tokens are revealed out of order and decisions are made under evolving context. 
This enables constraint enforcement to be integrated directly into the sampling process, without requiring additional trained models or modifying the diffusion model itself. 

At each denoising step, we modify token log-probabilities via an additive, constraint-dependent bias derived from adaptive Lagrangian multipliers, updated by mirror descent based on constraint violation. The bias is the optimal KL-regularised projection of the constraint at each step, so the resulting sampling distribution stays as close as possible to the model's unconstrained distribution while satisfying the constraint. In contrast to approaches that rely on auxiliary predictors or inner optimisation loops, this yields a lightweight procedure that retains the computational profile of standard sampling.

We evaluate the approach on three domains: topical text steering, molecular generation, and playlist generation. A single algorithm, instantiated with domain-specific scoring functions, improves constraint satisfaction across all three while preserving the relevant per-domain quality metrics. The KL cost of constraint enforcement is exposed as a single, controllable knob: practitioners can trade fidelity for satisfaction along a smooth frontier without retraining or per-step inner-loop optimisation.

\section{Related Work}

\paragraph{Discrete diffusion models.}
%
D3PM~\citep{austin2021d3pm} introduced discrete diffusion with structured transition matrices. We build on MDLM~\citep{sahoo2024mdlm}, which reduces this to absorbing-state (masking) dynamics with a continuous-time variational bound; SEDD~\citep{lou2024sedd} proposed alternative score-entropy training objectives for the same setting. Absorbing-state models naturally support hard positional constraints via inpainting (clamping a token at a known position and denoising the remainder), but inpainting alone cannot express soft or global targets such as ``molecular weight at least $350$'' or ``at least 10 ocean-related words''. Addressing such global, expectation-based constraints is the focus of our method.

\paragraph{Constrained generation.}
%
For continuous diffusion, classifier guidance~\citep{dhariwal2021diffusion} and classifier-free guidance (CFG)~\citep{ho2021classifierfree} steer generation toward desired attributes, but require either a trained classifier or retraining with label dropout, and operate in continuous state spaces that do not apply to discrete tokens. 
In autoregressive text generation, constrained decoding enforces lexical or structural constraints via grid beam search~\citep{hokamp2017lexical}, dynamic beam allocation~\citep{post2018fast}, look-ahead heuristics~\citep{lu2022neurologic}, or future discriminators~\citep{yang2021fudge}. These approaches rely on left-to-right generation and therefore do not apply to the parallel, partially observed denoising process of diffusion models. 

\paragraph{Constrained discrete diffusion.}
%
\citet{nisonoff2025unlocking} proposed predictor guidance for discrete diffusion, modulating transition rates via a Bayesian likelihood ratio from a trained predictor. This approach requires training a separate predictor per property and typically handles one property at a time. \citet{schiff2025simple} introduced discrete classifier-based guidance (D-CBG), which guides an unconditional denoising network using a separately trained classifier.
\citet{cardei2025cdd} introduced Constrained Discrete Diffusion (CDD), which avoids additional training but performs an inner Gumbel-softmax optimisation loop at each step, incurring $5$--$20\times$ computational overhead. In contrast, our method integrates constraint handling into the sampling process via a closed-form logit bias derived from a Lagrangian formulation, requiring no auxiliary models or inner optimisation. This yields a lightweight procedure with native multi-objective support and convergence guarantees from online convex optimisation~\citep{balseiro2020dual}.

More broadly, existing approaches either rely on auxiliary models, incur additional optimisation cost, or assume sequential generation, whereas our method addresses all three limitations simultaneously.\vspace{-0.09in}

\section{Notation \& Preliminaries}

We consider sequences of length $L$ over a vocabulary $\mathcal{V}=\{1,...,V\}$, and we introduce a special mask token $\texttt{M}$ and define the extended vocabulary $\tilde{\mathcal{V}}=\mathcal{V}\cup \{\texttt{M}\}$.
We let $T$ denote the discrete \textit{time horizon} of the diffusion process~\citep{sahoo2024mdlm}. At each step $t$, 
the sequence $x_t\in\tilde{\mathcal{V}}^L$ of length $L$ represents a partially masked version of the data, where we let $x_t^\ell$ be the token at position $\ell$.

Let $q(x_0)$ denote the data distribution. 
We define a forward process that progressively masks tokens in $x_0\sim q(x_0)$, independently across positions, increasing the probability that each token is replaced by the mask token $\texttt{M}$.


The reverse diffusion process reconstructs the original sequence by iteratively unmasking tokens. 
At each step $t$, unmasked tokens are sampled from the learned parametric model distribution $p_{\theta}(x | x_t)$.  
Generation starts from an initial sample $x_T$, typically taken as a fully masked sequence ($x_T=\{\texttt{M}\}_{\ell=1}^L$), and proceeds towards $x_0$. The goal is to recover a sample from the data distribution $q(x_0)$ by progressively removing noise.

We rely on two standard assumptions in discrete diffusion.

\begin{assumption} (Markov property) \label{as:markov}
    The model follows Markov factorisation:
    \begin{align}
    p_{\theta}(x_{0:T}) = p_\theta(x_T) \prod_{t=1}^T p_\theta(x_{t-1} | x_t).
\end{align}
\end{assumption}

\begin{assumption}(Absorbing state)\label{as:absorbing}
    For any step $t=1,...,T$, token $j\in\mathcal{V}$ and sequence $x_t$, if a position is already unmasked ($x_t^\ell \neq \texttt{M}$), then:
\begin{align}
    \Pr(x_{t-1}^\ell = j | x_t) = 
    \begin{cases}
        0\;\;& \forall j\neq x_t^\ell\\
        1\;\;&\text{otherwise.}
    \end{cases}
\end{align}
\end{assumption}
In other words, once a token is unmasked, it remains fixed for the remainder of the reverse diffusion process.
These assumptions are standard in discrete diffusion models~\citep{sahoo2024mdlm, austin2021d3pm}.

These properties define the standard reverse diffusion process used for generation. In the unconstrained setting, tokens are sampled directly from $p_\theta(x_{t-1} | x_t)$ at each step. In this work, we aim to modify this sampling process to enforce global constraints on the generated sequence while remaining close to the model's learned distribution. In the next section, we formalise this objective and derive an inference-time method for constrained generation.

\section{Method: Primal-Dual Guided Denoising}\label{sec:method}

Building on the diffusion framework described above, we now consider the constrained generation setting. We assume access to a trained diffusion model $p_\theta(x_{0:T})$ approximating the original data distribution $q(x_0)$ under Assumptions~\ref{as:markov} and~\ref{as:absorbing}. Our goal is to modify the generation process so that sampled sequences satisfy additional constraints, while remaining as close as possible to the model's learned distribution.


To formalise this, we introduce a target distribution $r(x_0)$ that incorporates the desired constraints. Rather than sampling directly from $q(x_0)$, we aim at sampling from $r(x_0)$, which is constrained but remains close to $q(x_0)$. This corresponds to adjusting the reverse diffusion process at each step while preserving its overall structure. 

We impose a Markov factorisation on $r(x_{0:T})$, mirroring the structure of the original model:
\begin{align}
    r(x_{0:T}) = p_\theta(x_T) \prod_{t=1}^T r(x_{t-1} | x_t),
\end{align}
We assume a shared prior $r(x_T) = p_\theta(x_T)$, and adopt the absorbing state Assumption~\ref{as:absorbing} for $r(x_{0:T})$.

At each step of the reverse diffusion process, $r(x_t^{\ell}=j | x_{t+1})$ defines the probability of assigning token $j$ to position $\ell$ given the current partially observed sequence $x_{t+1}$. 
We also introduce $a_{\ell j}^{t}\in[0,1]$, the probability that position $\ell$ is unmasked and assigned token $j$ at step $t$.
Formally,
\begin{align}\label{eq:alloc}
    a_{\ell j}^{t} = \mathbb{E}_{r(x_{t+1})}\Big[\mathbf{1}_{\{x_{t+1}^\ell = \texttt{M}\}} r(x_t^\ell = j | x_{t+1})\Big].
\end{align}
In what follows, we derive how this constrained distribution translates into a simple modification of token probabilities during sampling.

\paragraph{Constraints.}
We consider linear constraints on the generated sequence $x_0$, defined through per-token contributions. 
For clarity, we present the method for a single constraint: the extension to multiple constraints is straightforward and evaluated in Appendix~\ref{app:mol-details}.

Each token $j$ contributes a non-negative amount $b_{\ell j}\geq 0$ when assigned to position $\ell$. For example, in topical text generation, $b_{\ell j}$ can indicate whether token $j$ belongs to a target vocabulary, while in molecular generation it can correspond to atomic mass contributions. The total contribution of a sequence is obtained by summing these per-token values across positions. Given a target $R\geq 0$, we aim to generate sequences whose expected total contribution meets this target. 
Formally, 
\begin{align}
    \sum_{\ell=1}^L\sum_{j=1}^V b_{\ell j} r(x_0^{\ell}=j)\geq R.\label{eq:constr_1}
\end{align}

\paragraph{Objective.}
Given the constraint formulation above, our goal is to construct a distribution $r(x_{0:T})$ that satisfies the constraint while remaining as close as possible to the original model distribution $p(x_{0:T})$. This leads to the following optimisation problem:
\begin{align}
    \min_{r(x_{0:T})\in \Delta}\;& D_{KL}\big(r(x_{0:T})\, \big|\big| \,p_\theta(x_{0:T})\big)\label{eq:objective_KL}\\
    \text{subject to  }& \sum_{\ell=1}^L\sum_{j=1}^V b_{\ell j} r(x_0^{\ell}=j)\geq R.\label{eq:constr_expectation}
\end{align}
Intuitively, the objective encourages the constrained distribution $r$ to stay close to the learned model $p_{\theta}$, while the constraint enforces that generated sequences achieve the desired target. 


Noticing that the KL-divergence is equal to the difference between the cross-entropy and the entropy, we can rewrite Objective (\ref{eq:objective_KL}) in terms of token-level decisions (full derivations in Appendix~\ref{app:offline_problem}):
\begin{align}
    \max_{r(x_{0:T})\in \Delta} &\sum_{t=0}^{T-1} \sum_{\ell=1}^L \sum_{j=1}^V \mathbb{E}_{r(x_{t+1})}\Big[ r(x_{t}^{\ell}=j|x_{t+1})\log p_\theta(x_{t}^{\ell}=j|x_{t+1})  \Big] + H(r(x_{0:T}))\label{eq:offline_final_obj}\\
    \text{subject to  }& \sum_{\ell=1}^L\sum_{j=1}^V b_{\ell j} r(x_0^{\ell}=j)\geq R.\label{eq:offline_final_constr}
\end{align}
This formulation makes explicit that the objective decomposes across time steps and token positions, aligning with the structure of the diffusion process.

The formulation above is an \textit{offline} optimisation problem over full trajectories $r(x_{0:T})$. 
In practice, however, generation proceeds sequentially during the reverse diffusion process, where at each step we observe the current state $x_t$ and must decide how to sample the next tokens. 

This motivates an online formulation: at each step $t$, the model log probabilities $\log p_\theta(x_{t}^{\ell}=j|x_{t+1})$ are observed and fixed, and we choose the sampling distribution $r(x_{t} | x_{t+1})$ accordingly.  
Using the absorbing-state property and the definition of $a_{\ell j}^{t}$, we can rewrite the objective as:
\begin{align}
     \max_{r(x_{0:T})\in \Delta} \sum_{t=0}^{T-1} \sum_{\ell=1}^L \mathbb{E}_{r(x_{t+1})}\Big[\sum_{j=1}^V \mathbf{1}_{\{x_{t+1}^\ell=\texttt{M}\}}r(x_{t}^{\ell}=j|x_{t+1})\log p_\theta(x_{t}^{\ell}=j|x_{t+1})\Big] + H(r(x_{0:T})).
\end{align}
Since the logits are observed at each step, they can be treated as constants. Using Equation~(\ref{eq:alloc}), we can further express the objective in terms of the unmasking probabilities $a_{\ell j}^{t}$, yielding:
\begin{align}
    \max_{a_{\ell j}^t\in \Delta} &\sum_{t=0}^{T-1} \sum_{\ell=1}^L \sum_{j=1}^V a_{\ell j}^{t}\log p_\theta(x_{t}^{\ell}=j|x_{t+1}) + H(a_{\ell}^{t})\label{eq:objective_full} \\
    \text{subject to }& \sum_{t=0}^{T-1}\sum_{\ell=1}^L\sum_{j=1}^V b_{\ell j} a_{\ell j}^{t}\geq R.\label{eq:constraint_full}
\end{align}
This yields an online optimisation problem over token-level sampling decisions. Importantly, this formulation shows that constrained generation can be implemented by modifying token-level sampling probabilities at each denoising step. We now derive how this takes the form of a simple logit bias using a primal-dual formulation.

\subsection{Primal-Dual Inference}

To solve the online optimisation problem (\ref{eq:objective_full})-(\ref{eq:constraint_full}), we aim to derive a sampling rule that modifies token probabilities at each denoising step in a way that enforces the constraint while remaining close to the model distribution. We achieve this using a primal-dual formulation.
We begin by dualising constraint (\ref{eq:constraint_full}), yielding:
\begin{align}
    \max_{a_{\ell j}^t\in \Delta} \Big[\sum_{t=0}^{T-1} \sum_{\ell=1}^L \sum_{j=1}^V a_{\ell j}^{t}\log p_\theta(x_{t}^{\ell}=j|x_{t+1}) + H(a_{\ell}^t) + \min_{\lambda \geq 0} \lambda \Big( \sum_{t=0}^{T-1}\sum_{\ell=1}^L\sum_{j=1}^V b_{\ell j} a_{\ell j}^t - R\Big) \Big].\label{eq:dualized}
\end{align}
For any fixed value of the Lagrangian multiplier $\lambda=\lambda_{t+1}$, the first-order optimality conditions yield the following sampling distribution:
\begin{align}
    a_{\ell j}^t = \frac{p_\theta(x_{t}^\ell=j|x_{t+1}) e^{ \lambda_{t+1} b_{\ell j}}}{\sum_{i=1}^V p_\theta(x_{t}^\ell=i|x_{t+1}) e^{ \lambda_{t+1} b_{\ell i}}}.\label{eq:sampling}
\end{align}
This corresponds to modifying the model probabilities by an exponential weighting with the dual-adjusted terms $\lambda_{t+1} b_{\ell i}$. 
This admits an intuitive energy-based interpretation. 
The dual term $\lambda_{t+1} b_{\ell i}$ acts as an external field on the unbiased model's $p_\theta$ energy landscape: it makes the constraint-favoring tokens ``energetically cheaper'', guiding the reverse process to these new low energy states.


This shows that constrained generation can be implemented by simply adjusting the log probabilities at each denoising step, without altering the model itself. Intuitively, tokens that contribute more toward satisfying the constraint are upweighted, while others are relatively suppressed.

As generation proceeds, the degree of constraint satisfaction may change.
Therefore, the multiplier $\lambda$ must be updated accordingly. 
We perform this update using mirror descent (derivation in Appendix \ref{app:mirror_descent}):
\begin{equation}\label{eq:mirror}
    \lambda_t = \lambda_{T} \exp\Big\{ -\eta \sum_{t=0}^{T-1}\sum_{\ell=1}^L\sum_{j=1}^V \Big(b_{\ell j}-\frac{R}{L}\Big) a_{\ell j}^t \Big\}.
\end{equation} 
In words, the update tracks the cumulative constraint satisfaction during generation. When the constraint is under-satisfied, the multiplier increases, strengthening the bias toward high-contribution tokens; when the constraint is satisfied or exceeded, the multiplier decreases, reducing interference with the model distribution.

Algorithm~\ref{alg:pd} summarises the resulting procedure. Starting from a fully masked sequence $x_T$, we iteratively perform the reverse diffusion steps. At each step, we first compute the modified sampling probabilities using Equation~(\ref{eq:sampling}), then sample tokens for the masked positions, update the running constraint satisfaction, and adjust the Lagrangian multiplier using Equation~(\ref{eq:mirror}). This yields a simple inference-time procedure that integrates constraint enforcement directly into the denoising process through a logit-level modification of the sampling distribution.
Importantly, this modification is applied entirely at inference time and does not require retraining or additional model evaluations.


\begin{algorithm}[t]
\caption{Primal-Dual Guided Denoising for MDLM}
\label{alg:pd}
\begin{algorithmic}[1]
\REQUIRE Model $p_\theta(x_{0:T})$, parameters $(R, \eta, b_{\ell j})$, \#steps $T$, initial $\lambda_T$
\STATE $g_T \leftarrow 0$;\quad $x_T \leftarrow \texttt{[M]}^{L}$
\FOR{$t = T-1, \ldots, 0$}

    
    \STATE $a_{\ell j}^t \leftarrow p_\theta(x_{t} | x_{t+1}) e^{\lambda_{t+1} b_{\ell j}}\;\;\;\;\;\forall \ell=1,...,L, \;\;\forall j=1,...,V$ \hfill {\small\textit{forward pass}}

    
    \STATE $a_{\ell j}^t \leftarrow a_{\ell j}^t / \sum_{i=1}^V a_{\ell i}^t$ \hfill {\small\textit{see} Eq. (\ref{eq:sampling})}
    
    \STATE Sample $x_t^\ell \sim Cat(\cdot \mid a_\ell^t)$ for all $\ell=1,...,L$; update masked positions
    

    \STATE $g_t \leftarrow g_{t+1} + \sum_{\ell=1}^L \Big(b_{\ell, x_t^\ell}-\frac{R}{L}\Big) \mathbf{1}_{x_t^\ell \neq \texttt{[M]} \;\text{\&}\; x_{t+1}^\ell = \texttt{[M]}}$ \hfill {\small\textit{slack update}}

    \STATE $\lambda_t \leftarrow \lambda_{T} \exp\Big\{ -\eta g_{t} \Big\}$ \hfill {\small\textit{see}  Eq. (\ref{eq:mirror})}

\ENDFOR
\RETURN $x_0$
\end{algorithmic}
\end{algorithm}

\subsection{Theoretical Analysis} 
Having introduced the primal-dual decoding algorithm, we now turn to its theoretical analysis.
In particular, we bound the constraint violation reached by Algorithm~\ref{alg:pd} and 
characterise the resulting trade-off between constraint satisfaction and approximation error. 
We also discuss the computational complexity of the proposed algorithm.

\begin{restatable}{statement}{OnlineAlgCV}\label{th:cv}
    If Problem (\ref{eq:objective_full})-(\ref{eq:constraint_full}) is feasible, then for any $\eta>0$, the expected constraint violation of Algorithm \ref{alg:pd} is upper-bounded by $\mathcal{O}\Big(\frac{1}{\eta} \log \frac{\eta+1}{b_{max}-\frac{R}{L}} \Big)$.
\end{restatable}
\begin{proof}
    We provide the proof in Appendix \ref{app:cv_bounds}.
\end{proof}
This result shows that increasing $\eta$  tightens the bound on constraint violation. Intuitively, from Equation~(\ref{eq:mirror}), a larger $\eta$ leads to stronger updates the Lagrangian multiplier $\lambda_{t+1}$ when the constraint is not satisfied, which increases the sampling probability of tokens with higher contribution $b_{\ell j}$.

We now introduce a $\textit{weak temporal consistency}$ assumption that models how token-level log probabilities evolve across denoising steps. 
In particular, we decompose the logits produced by the diffusion model into a stable component and a small stochastic drift: 
\begin{assumption}\label{as:temporal_consistency}
    For any token $j\in\mathcal{V}$, step $t=T-1,...,0$, position $\ell=1,...,L$ and sequence $x_{t+1}$, we assume
    \begin{align}
        \log p_\theta(x_t^\ell=j | x_{t+1}) = \log p_\theta(x_{t+1}^\ell=j | x_{t+2}) + \epsilon_{t j \ell},
    \end{align}
    where $\epsilon_{tj\ell}\sim\mathcal{N}(\mu_{tj\ell},\sigma^2)$, with bounded covariance $Cov(\epsilon_{tj\ell}, \epsilon_{\tau jk})\leq\rho$ 
    for all $(t,\ell)\neq (\tau, k)$, and bounded drift $|\mu_{tj\ell}|\leq \bar{\mu}$.
\end{assumption} 
Intuitively, this assumption states that the model's log probabilities evolve smoothly across denoising steps, with changes captured by a small stochastic perturbation. This reflects the fact that predictions become progressively more refined as the reverse diffusion process unfolds, rather than changing abruptly. Similar temporal consistency assumptions have been used in the analysis of sequential generative models, including LLMs~\citep{maystre2025incrementalsequenceclassificationtemporal}.


\begin{restatable}{statement}{OnlineAlgRegret}\label{th:regret}
    The regret of Algorithm \ref{alg:pd} is upper bounded by 
    \begin{align}\label{eq:theorem_regret}
        Regret_\epsilon(\mathcal{A}) \leq \frac{\lambda_T}{\eta} e^{\eta L T \max_{\ell,j}|b_{\ell j} - R/L|}+C_\delta,
    \end{align}
    where $C_\delta = \bar{\mu} LT(T+1) + \sqrt{4(\sigma^2 + LT\rho) \log\frac{2LVT(VT+1)^L}{\delta} (LT+L^2T^3)}$ with probability at least $1-\delta$.
\end{restatable}
\begin{proof}
    We provide the proof in Appendix~\ref{app:reg_bounds}.  
\end{proof} 
This bound makes explicit the trade-off between constraint satisfaction and approximation error. In particular, increasing the step size $\eta$ tightens the constraint violation bound (Statement \ref{th:cv}), but leads to a larger regret term in Equation~(\ref{eq:theorem_regret}). This reflects the tension between enforcing constraints aggressively and remaining close to the original model distribution. This trade-off is reflected empirically in Section~\ref{sec:experiments}, where $\eta$ controls the balance between constraint satisfaction and distributional fidelity.



The second term $C_\delta$ captures stochastic variability of the model logits, as described in Assumption \ref{as:temporal_consistency}. When the variance $\sigma^2$, covariance $\rho$, and drift $\bar{\mu}$ are small, this term decreases, tightening the regret bound. We provide an empirical study of Assumption \ref{as:temporal_consistency} and its effect on the bounds in Appendix~\ref{app:text-temporal-consistency}.

Finally, the proposed primal-dual decoding adds only constant-time operations to the inference loop of the unconstrained model (lines 3-4, 6-7 of Algorithm~\ref{alg:pd}), and therefore preserves the computational complexity of standard discrete diffusion models~\citep{sahoo2024mdlm}. 

\subsection{Extensions: Lagrangian Update Regimes}\label{sec:method-lambda}
While Constraint (\ref{eq:constraint_full}) enforces a global target $R$ over the full sequence of length $L$, the Lagrangian update in Equation~(\ref{eq:mirror}) 
distributes this target uniformly across positions via the per-position term 
$R/L$. This provides a stable normalised signal during generation, but may be suboptimal when constraint satisfaction is unevenly distributed across tokens or steps. 

To address this, we consider alternative ways of measuring constraint satisfaction during the reverse diffusion process, leading to different update regimes for the Lagrangian multiplier. Specifically, we study three variants:
(1) \emph{Accumulated} slack follows Equation~(\ref{eq:mirror}), aggregating constraint contributions over time and providing conservative, stable target tracking. 
(2) \emph{Instantaneous} slack compares the accumulated contribution directly to the global target $R$ , rather than the per-position target $R/L$, yielding 
$\lambda_t=\lambda_{T} \exp\{ -\eta (\sum_{\tau=t}^{T-1}\sum_{\ell=1}^L\sum_{j=1}^V b_{\ell j} a_{\ell j}^\tau -R)\}$. 
Finally, (3) \emph{Optimistic slack}, based on Optimistic Mirror Descent (OMD)~\citep{rakhlin2013online}, incorporates a prediction of future constraint satisfaction:
\begin{equation}\label{eq:optimistic}
    \lambda_t = \lambda_T \exp\{-\eta \sum_{\tau=t}^{T-1} \Big[\Big( \sum_{\ell=1}^L\sum_{j=1}^V b_{\ell j} a_{\ell j}^\tau -R \Big) -M_\tau + M_{\tau+1}\Big] \},
\end{equation}
where $M_\tau = R - \sum_{\ell=1}^L\sum_{j=1}^V p_\theta(x_\tau^\ell = j | x_{\tau+1}) b_{\ell j} \mathbf{1}_{\{x_\tau^\ell = \texttt{[M]}\}}$ is the expected future shortfall estimated from the model's predictions.


Intuitively, accumulated slack provides stable but conservative updates, instantaneous slack enforces constraints aggressively, and optimistic slack balances the two by anticipating future contributions.

\section{Experiments}\label{sec:experiments}

We evaluate primal-dual guided decoding on three structurally different discrete-sequence domains: a sparse counting constraint (topical text), a dense additive property under hard syntactic constraints (molecular generation), and a noisy indirect cluster-score constraint (playlist generation). For each domain we train a single MDLM backbone once and keep it fixed across all diffusion methods. 

Table~\ref{tab:unified} reports two cross-domain metrics. \emph{Pass\%} measures constraint satisfaction (e.g., $\geq R$ ocean tokens for text, MW $\geq 350$ for molecular generation, and fraction of generated tracks tagged with the target subgenre for playlist). \emph{KL} measures distributional divergence as the unigram KL between constrained and unconstrained samples from the same model. Together, these metrics capture the trade-off between constraint satisfaction and fidelity to the original model distribution.
Domain-specific quality metrics and full hyperparameter sweeps are provided in Appendix~\ref{sec:extended-results}.

We compare against three diffusion baselines that share the same backbone: \emph{Static logit bias} (a fixed scalar $\alpha$ added to per-token scores at every step), \emph{D-CBG}~\citep{schiff2025simple} (discrete classifier-based guidance using a separately trained property classifier), and \emph{CDD}~\citep{cardei2025cdd} (training-free with an inner Gumbel-softmax loop per step). The appendix tables additionally include comparisons with \emph{GPT-4.1-mini}~\citep{achiam2023gpt4} on text and molecular, and zero-shot transfer of SPDD to \emph{LLaDA-8B}~\citep{nie2024llada} on text. Each method is reported at a representative operating point; the satisfaction--KL trade-off is monotonic in each method's primary control parameter, so no held-out tuning set is required.

\begin{table}[t]
\centering
\caption{Comparison between SPDD and baselines across three datasets, with relative inference cost.}
\label{tab:unified}
\begin{tabular}{l cc cc cc cc}
\toprule
 & \multicolumn{2}{c}{Text} & \multicolumn{2}{c}{Molecular} & \multicolumn{2}{c}{Playlist} & \multicolumn{2}{c}{Cost} \\
\cmidrule(lr){2-3} \cmidrule(lr){4-5} \cmidrule(lr){6-7} \cmidrule(lr){8-9}
Method & Pass\% & KL$\downarrow$ & Pass\% & KL$\downarrow$ & Pass\% & KL$\downarrow$ & Sample$\times$ & Train \\
\midrule
Unconstrained       & \phantom{0}0.50\% &     --- & 29.40\% &     --- & 1.14\% &     --- & 1.00$\times$ & none \\
Static bias         & 81.90\% & 0.706 & 41.40\% & 0.019 & 1.72\% & 0.003 & 1.03$\times$ & none \\
D-CBG               & \phantom{0}0.00\% & 0.127 & 42.80\% & 0.035 & 1.92\% & 0.115 & 1.09$\times$ & classifier \\
CDD                 & 38.00\% & 0.221 & 42.20\% & 0.034 & 9.18\% & 0.779 & 7.28$\times$ & none \\
\midrule
SPDD (Ours)         & 92.10\% & 0.420 & 48.50\% & 0.096 & 9.28\% & 0.355 & 1.03$\times$ & none \\
\bottomrule
\end{tabular}
\end{table}

SPDD reaches the highest Pass\% on every domain. On text it more than doubles CDD's satisfaction; on molecular it lifts the upper-tail acceptance several points beyond the strongest baseline; on playlist it matches CDD's track-level satisfaction at less than half the KL, while preserving substantially higher token diversity (Appendix~\ref{app:playlist-exp-details}). The corresponding KL is the cost of pushing further into the constrained region, and is itself a controllable knob: $\eta$ traces a smooth frontier between satisfaction and fidelity (Appendix~\ref{sec:extended-results}). These gains come at near-baseline sampling speed and without auxiliary classifier training. The remainder of this section analyses the domain-specific mechanisms underlying these results.

\paragraph{Text results.}\label{sec:text}

We evaluate topical steering in text generation by generating children's stories enriched with ocean-related vocabulary, without retraining the underlying model. This tests whether the method can enforce sparse, global lexical constraints while preserving fluency and style. We train a 142M-parameter MDLM on 50k stories from TinyStories~\citep{eldan2023tinystories} (BPE tokenizer, vocabulary 8{,}192, sequence length 256), and constrain generation to include at least $R{=}10$ ocean-related tokens per story, drawn from a set of 18 target words (e.g., \emph{ocean, whale, dolphin}).

Static bias improves satisfaction but overshoots, as it cannot reduce pressure once the constraint is satisfied. CDD improves over unconstrained sampling but requires an inner optimisation loop. D-CBG underperforms on this sparse target set: its Taylor-approximated classifier gradient is too smooth to produce a sharp bias over the 18 target words, redistributing mass within the non-ocean vocabulary without raising any target token above the sampling threshold, so Pass\% remains near the unconstrained level. In contrast, the primal-dual formulation adapts its multiplier to the accumulated constraint deficit, allowing SPDD to increase satisfaction while staying much closer to the unconstrained distribution.

The same mechanism transfers zero-shot to LLaDA-8B without fine-tuning, with the caveat that the larger backbone exhibits distinct qualitative artefacts (Appendix~\ref{app:llada}).

The slack definition controls how aggressively $\lambda$ responds to the running constraint deficit. Figure~\ref{fig:lambda_trace} illustrates the dynamics of the accumulated slack used in the main experiments: $\lambda$ rises as the sample falls behind its target and decreases as new ocean tokens reduce the deficit. Sweeping the slack mode (see Section~\ref{sec:method-lambda}) and $\eta$ traces a continuum between conservative low-KL guidance and more aggressive high-satisfaction guidance (Appendix~\ref{app:text-slack}).

\begin{figure}[t]
\centering
\includegraphics[width=0.7\linewidth]{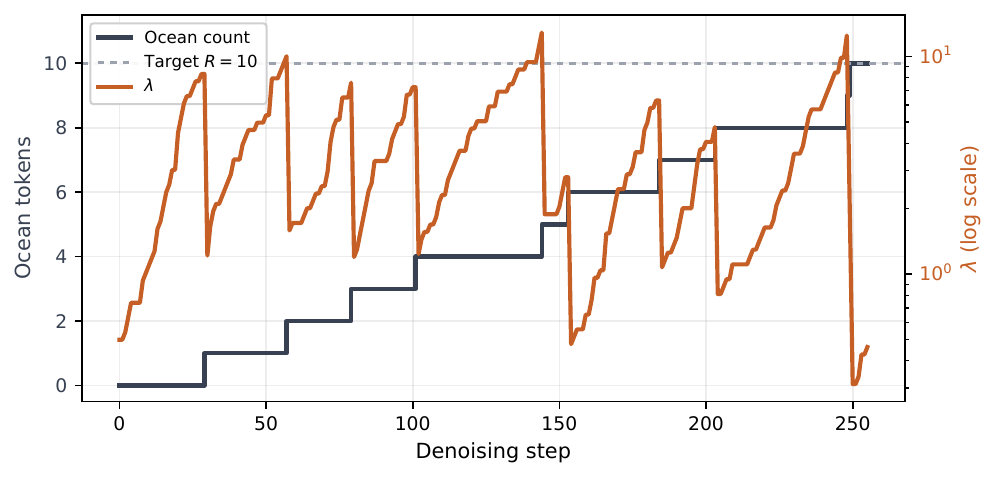}
\caption{Accumulated-slack multiplier dynamics for a single TinyStories sample ($\eta\!=\!2.0$, target $R\!=\!10$). Ocean-token count and the target are shown on the left axis; the coloured curve shows $\lambda$ on the right axis, rising as the accumulated deficit grows and stepping down when newly committed ocean tokens reduce it. The sample reaches the target by the end of denoising.}
\label{fig:lambda_trace}
\end{figure}

\paragraph{Molecular results.}\label{sec:mol}

We evaluate constrained molecular generation: given a diffusion model trained on drug-like molecules, we steer generation toward higher molecular weight (MW) while preserving chemical validity and diversity. This tests whether the method can control a global numerical property without breaking structural constraints. We train a 21.9M-parameter MDLM on 1M molecules from ZINC20~\citep{irwin2020zinc20}, represented as SMILES strings, and enforce MW~$\geq 350$ using atomic-mass per-token scores. The above-target threshold is well into the upper tail of the unconstrained MW distribution (mean MW $\approx 330$), so the constraint exercises non-trivial steering.

In this domain the unconstrained model already places non-negligible mass above the target, so the question is how much further guidance can shift the distribution without breaking chemistry. SPDD attains the highest Pass\% in the column while keeping validity above 50\%. D-CBG and CDD move closer to the threshold but plateau several points below SPDD; static logit bias becomes competitive only at high $\alpha$, where the lack of an adaptive feedback loop already pushes mean MW well past the target. The KL of SPDD is correspondingly larger because it shifts more mass into the constrained region; the satisfaction--KL frontier is exposed via $\eta$ (Appendix~\ref{app:mol-details}).

\paragraph{Playlist results.}\label{sec:playlist}

We apply SPDD to music playlist generation, steering a diffusion model to complete playlists with tracks from a target subgenre. The model generates discrete tokens mapped back to tracks. Each token is an LSH cluster over audio embeddings, and its genre score is the fraction of cluster members tagged with the target genre.

We evaluate playlists in a held-out completion setting: the first 20 session tracks are clamped and the model generates the remaining 30. Pass\% measures track-level subgenre satisfaction, the fraction of generated tracks tagged with the target subgenre.

The playlist signal is sparse and indirect: each token is an LSH cluster whose members only probabilistically exhibit the target subgenre, so guidance must amplify a weak per-token signal into a non-trivial output-distribution shift. Static is essentially passive. D-CBG, despite a trained subgenre classifier, barely lifts track-level satisfaction: its gradient is too smooth to raise any cluster's probability above the sampling threshold, the same effect we observed on the sparse text constraint. CDD achieves non-trivial track-level satisfaction but at the cost of substantial token-diversity collapse. SPDD improves on CDD satisfaction level at less than half the KL while keeping token diversity much closer to the unconstrained baseline; the remaining trade-off is on relevance to the user's natural continuation, which any meaningful steering of the distribution must shift. 

This trade-off is itself controllable: varying $\eta$ moves SPDD along the satisfaction--KL frontier (Appendix~\ref{app:playlist-exp-details}), letting the user pick a point that fits a target fidelity or satisfaction budget. In contrast, CDD's $\tau$ does not provide a comparable increase in satisfaction in our sweep, and D-CBG's $\gamma$ saturates at the classifier's decision boundary.
A consistent pattern emerges across the three domains. When the constraint signal is sparse, indirect, or mediated by noisy token clusters, the primal-dual update provides the most favourable trade-off: it improves satisfaction while preserving the base sampler's computational profile and bounding distributional fidelity through $\eta$.

\section{Discussion}

The central contribution of primal-dual guided decoding is \emph{enabling constrained generation with minimal distributional divergence}. The additive logit bias arises as the optimal solution to a KL-regularised Lagrangian (Equation~\ref{eq:sampling}), so guidance operates directly in the model's probability space rather than through re-prompting, auxiliary predictors, or inner optimisation. Empirically, SPDD keeps constrained samples close to each model's unconstrained distribution across text, molecular, and playlist generation (Table~\ref{tab:unified}). The comparisons clarify the trade-offs of the alternatives: CDD is a training-free option when additional inner-loop computation is acceptable, and D-CBG offers a classifier-driven path when a property classifier is already trained. SPDD complements these methods by targeting decomposable constraints with an adaptive, low-cost update. The formal violation bound (Statement~\ref{th:cv}) makes the satisfaction--fidelity trade-off explicit through $\eta$.

The framework is lightweight: it requires no retraining, no auxiliary model, and only a per-token score vector. In the main experiments we use the accumulated slack update, which follows the mirror-descent derivation directly; alternative slack definitions are evaluated in the appendix.

\paragraph{Limitations.}
Our experiments focus on absorbing-state diffusion, where early token commitments are not revisited, making slack design and $\eta$ scheduling most important in early denoising steps (Table~\ref{tab:ocean-slack}). Extending the same primal-dual principle to non-absorbing discrete diffusion is a promising direction for enabling later corrections. For non-decomposable constraints (e.g., QED), improved partial-sequence slack estimators could enhance early updates, and very large backbones may benefit from standard repetition controls or adaptive $\lambda_0$ schedules. Finally, the method provides probabilistic constraint bounds (Statement~\ref{th:cv}); applications requiring strict feasibility may combine it with rejection sampling or post-hoc checks.

\newpage
\bibliographystyle{plainnat}
\bibliography{references}

\appendix
\section{Additional Proofs}\label{app:proofs}

\subsection{Chain Rule for KL divergence, Equation~(\ref{eq:chain_KL})}\label{app:chain_rule_KL}
\begin{align} 
    D_{KL}\Big( &p(x_T) \prod_{t=1}^T r(x_{t-1} | x_t)\, \Big|\Big| \,p(x_T) \prod_{t=1}^T p_\theta(x_{t-1} | x_t)\Big) = \nonumber\\ 
    &\sum_{x_0,...,x_T} p(x_T) \prod_{t=1}^T r(x_{t-1} | x_t) \log \frac{ \prod_{t=1}^T r(x_{t-1} | x_t)}{\prod_{t=1}^T p_\theta(x_{t-1} | x_t)}=\nonumber\\
    &\sum_{x_0,...,x_T} p(x_T) \prod_{t=1}^T r(x_{t-1} | x_t) \sum_{t=1}^T \log \frac{r(x_{t-1}|x_t)}{p_\theta(x_{t-1}|x_t)} = \nonumber\\
    &\sum_{x_T} p(x_T) \underbrace{\sum_{x_{T-1}} r(x_{T-1}|x_T)\log \frac{r(x_{T-1}|x_T)}{p_\theta(x_{T-1}|x_T)}}_{D_{KL}\big(r(x_{T-1}|x_T) || p_\theta(x_{T-1}|x_T)\big)}\underbrace{\sum_{x_{0:T-2}}\prod_{t=1}^{T-1} r(x_{0:T-2}|x_{T-1})}_{=1} + \nonumber\\
    \sum_{x_{T-1:T}} & r(x_{T-1}|x_T)p(x_T) \underbrace{\sum_{x_{T-2}} r(x_{T-2}|x_{T-1})\log \frac{r(x_{T-2}|x_{T-1})}{p_\theta(x_{T-2}|x_{T-1})}}_{D_{KL}\big(r(x_{T-2}|x_{T-1}) || p_\theta(x_{T-2}|x_{T-1})\big)}\underbrace{\sum_{x_{0:T-3}} \prod_{t=1}^{T-2} r(x_{0:T-3}|x_{T-2})}_{=1}+ ...= \nonumber\\ 
    \sum_{t=0}^{T-1} & \mathbb{E}_{r(x_{t+1})}\Big[ D_{KL}(r(x_{t}|x_{t+1}) || p_\theta(x_{t}|x_{t+1}))\Big].
\end{align}
  
\subsection{Derivation of Equations (\ref{eq:offline_final_obj})-(\ref{eq:offline_final_constr})}\label{app:offline_problem}
We can re-write the KL-divergence term in Objective (\ref{eq:objective_KL}) as follows:
\begin{align}
    D_{KL}\big(r(x_{0:T})\, \big|\big| \,p_\theta(x_{0:T})\big) = & D_{KL}\Big(p(x_T) \prod_{t=0}^{T-1} r(x_{t} | x_{t+1})\, \Big|\Big| \,p(x_T) \prod_{t=0}^{T-1} p_\theta(x_{t} | x_{t+1})\Big) =\\
    &\sum_{t=0}^{T-1} \mathbb{E}_{r(x_{t+1})} D_{KL}\Big( r(x_{t} | x_{t+1})\, \big|\big| \, p_\theta(x_{t} | x_{t+1})\Big)=\label{eq:chain_KL}\\
    &\sum_{t=0}^{T-1} \sum_{\ell=1}^L \mathbb{E}_{r(x_{t+1})} D_{KL}\Big( r\big(x_{t}^{\ell} | x_{t+1}\big)\, \big|\big| \, p_\theta\big(x_{t}^{\ell} | x_{t+1}\big)\Big),\label{eq:independence}
\end{align}
where Equation~(\ref{eq:chain_KL}) follows from the chain rule for KL divergence (see Appendix \ref{app:chain_rule_KL}) and Equation~(\ref{eq:independence}) follows from independence of sampling across positions $\ell=1,...,L$ in the generated sequences.
Expanding the KL-term we can further simplify this optimisation objective as follows:
\begin{align}
    \min_{r(x_{0:T})\in \Delta}\; &D_{KL}\big(r(x_{0:T})\, \big|\big| \,p_\theta(x_{0:T})\big) =\nonumber \\
    &\min_{r(x_{0:T})\in \Delta}\; \sum_{t=0}^{T-1} \sum_{\ell=1}^L \mathbb{E}_{r(x_{t+1})} \sum_{j=1}^V r(x_{t}^{\ell}=j|x_{t+1})\log\Bigg( \frac{r(x_{t}^{\ell}=j|x_{t+1})}{p_\theta(x_{t}^{\ell}=j| x_{t+1})}\Bigg)=\nonumber\\
    &\max_{r(x_{0:T})\in \Delta} \sum_{t=0}^{T-1} \sum_{\ell=1}^L \mathbb{E}_{r(x_{t+1})}\Big[\sum_{j=1}^V r(x_{t}^{\ell}=j|x_{t+1})\log p_\theta(x_{t}^{\ell}=j|x_{t+1})  \Big] + H(r(x_{0:T})).\label{eq:obj_entropy}
\end{align}

\subsection{Mirror Descent}\label{app:mirror_descent}
To derive the update rule for $\lambda_t$ we rely on the mirror descent applied to the inner minimisation problem in Equation~(\ref{eq:dualized}).
In particular, we let
\begin{align}\label{eq:mirror_desc}
    \lambda_t = \arg\min_{\lambda\geq 0}\Big\{\sum_{t=0}^{T-1}\sum_{\ell=1}^L\sum_{j=1}^V \big(b_{\ell j}-\frac{R}{L}\big) a_{\ell j}^t + D_h(\lambda||\lambda_{t+1})\Big\},
\end{align}
where $D_h(\lambda||\lambda_{t+1})$ is the Bregman divergence that is induced by the negative entropy function $h(\lambda) = \lambda \ln\lambda - \lambda$.
Solving Problem (\ref{eq:mirror_desc}) we obtain:
\begin{align}
    \lambda_t = \lambda_{t+1} \exp\Big\{ -\eta \sum_{\ell=1}^L\sum_{j=1}^V \big(b_{\ell j}-\frac{R}{L}\big) a_{\ell j}^t \Big\},
\end{align} 
where $\eta>0$ is a calibration step-size parameter.
Equivalently, we obtain:
\begin{align}
    \lambda_t = \lambda_{T} \exp\Big\{ -\eta \sum_{t=0}^{T-1}\sum_{\ell=1}^L\sum_{j=1}^V \big(b_{\ell j}-\frac{R}{L}\big) a_{\ell j}^t \Big\}.
\end{align} 

\subsection{Expected Regret Analysis}\label{app:reg_bounds}
In this section, we derive an upper bound on the optimal reward. 
To this end, we start with analysing the stochastic process defined in Assumption \ref{as:temporal_consistency}.
Specifically, from Assumption \ref{as:temporal_consistency} we can write:
\begin{align}\label{eq:logits_process}
        \log p_\theta (x_t^\ell = j | x_{t+1}) = \log p_\theta(x_T^\ell = j) + \sum_{\tau=t}^{T-1} \epsilon_{\tau j\ell},\;\;\;\forall t=1,...,T-1, \ell=1,...,L, j\in\mathcal{V},
\end{align}
where the logits computed at time step $t$ of the reverse diffusion process are modelled as a stochastic process that steers the prior logits $\log p_\theta(x_T^\ell = j)$ with a Gaussian process $\sum_{\tau=t}^{T-1} \epsilon_{tj\ell}$.
We decompose $\epsilon_{tj\ell} = \mu_{tj\ell} + \eta_{tj\ell}$ into the random noise $\eta_{tj\ell}\sim\mathcal{N}(0,\sigma^2)$ and the constant average drift $\mu_{tj\ell}\in \mathbb{R}$. 

Now, let $\mathcal{H}_{tj\ell} = \sum_{\tau=t}^{T-1} \eta_{\tau j\ell}$ and $\mathcal{M}_{tj\ell} = \sum_{\tau=t}^{T-1} \mu_{\tau j\ell}$ be the cumulative random and the drift components reached by step $t$ of the reverse process. 
Notice, that $\mathcal{H}_{tj\ell}$ is a zero-mean process for $t=T-1,...,0$, s.t.,
\begin{align}
    Var(\mathcal{H}_{tj\ell}) \leq (T-t)\sigma^2 + (T-t)(T-t-1)\rho, 
\end{align}
and
\begin{align}
    Cov(\mathcal{H}_{tj\ell}, \mathcal{H}_{tjk}) \leq (T-t)^2\rho\;\;\;\forall \ell\neq k, \forall t=1,...,T, \forall j\in\mathcal{V}.
\end{align}
Furthermore, using the fact that $\mu_{tj\ell} \leq \bar{\mu}$ (see Assumption \ref{as:temporal_consistency}) we can bound 
\begin{align}\label{eq:bound_drift}
    \mathcal{M}_{tj\ell}\leq \bar{\mu}\cdot (T-t).
\end{align}

We let $Reward(a, \epsilon)$ be the reward reached by the randomised allocation $a\sim a^*(\epsilon)$ drawn from the optimal policy $a^*(\epsilon)$ given randomness $\epsilon$ defined in Assumption \ref{as:temporal_consistency}.
Observe, that the expected reward $\mathbb{E}_{a\sim a^*(\epsilon)} [Reward(a, \epsilon)\, |\, \epsilon\,]$ reached with this randomised allocation $a$ is itself a random variable as it depends on $\epsilon$. 
We will, therefore, derive an upper bound for this expected reward that must hold with high probability given the randomness generated by $\epsilon$.
\begin{statement}\label{th:reward_OPT}
    Expected reward $\mathbb{E}_{a\sim a^*(\epsilon)} [Reward(a, \epsilon)\, |\, \epsilon\,]$ of the optimal allocation in Problem (\ref{eq:objective_full})-(\ref{eq:constraint_full}) is upper bounded by 
    \begin{align}\label{th:opt_bound}
        \mathbb{E}_{a\sim a^*(\epsilon)} [Reward(a, \epsilon)\, |\, \epsilon\,] \leq \min_{\lambda\geq 0}\; & T \sum_{\ell=1}^L \log\sum_{i=1}^V p_\theta (x_T^\ell=i) e^{\lambda(b_{\ell i} - R/L)}\nonumber\\
    &  + \bar{\mu} LT + \sqrt{2LT(\sigma^2 + LT\rho)( \log\frac{1}{\delta} + L\log(VT+1))}
\end{align}
with probability at least $1-\delta$.
\end{statement}
\begin{proof}
From Equations~(\ref{eq:objective_full}) and (\ref{eq:logits_process}) we obtain:
\begin{align}
    \sum_{t=0}^{T-1}\sum_{\ell=1}^L\sum_{j=1}^V a_{\ell j}^t \log p_\theta (x_t^\ell = j | x_{t+1}) = &\sum_{t=0}^{T-1}\sum_{\ell=1}^L\sum_{j=1}^V a_{\ell j}^t\log p_\theta(x_T^\ell = j) +\nonumber \\
    &\underbrace{\sum_{t=0}^{T-1}\sum_{\ell=1}^L\sum_{j=1}^V a_{\ell j}^t \mathcal{M}_{tj\ell}}_{\text{systematic drift}} + \underbrace{\sum_{t=0}^{T-1}\sum_{\ell=1}^L\sum_{j=1}^V a_{\ell j}^t \mathcal{H}_{tj\ell}}_{\text{random component}}.
\end{align}
Remember that in the MDLM setting once the token is allocated (i.e., $a_{\ell j}^t=1$), it stays fixed until the end of generation, i.e.,
\begin{align}\label{eq:feasibility}
    \sum_{t=0}^{T-1} \sum_{j=1}^V a_{\ell j}^t \leq 1\;\;\;\;\forall \ell=1,...,L.
\end{align}
Therefore, using Equations (\ref{eq:bound_drift}) and (\ref{eq:feasibility}) the systematic drift can be upper bounded as follows:
\begin{align}
    \sum_{t=0}^{T-1}\sum_{\ell=1}^L\sum_{j=1}^V a_{\ell j}^t \mathcal{M}_{tj\ell} \leq \sum_{t=0}^{T-1}\sum_{\ell=1}^L\sum_{j=1}^V a_{\ell j}^t\bar{\mu}\cdot (T-t)\leq \bar{\mu}LT.
\end{align}

Now, for any position $\ell$ let $\mathcal{H}_{\ell t}=\sum_{j=1}^V a_{\ell j}^t \mathcal{H}_{tj\ell}\sim \mathcal{N}\big(0,\, Var(\mathcal{H}_{\ell t})\big)$, where
\begin{align}
    Var(\mathcal{H}_{\ell t}) = \sum_{i,j=1}^V  a_{\ell j}^t a_{\ell i}^t Cov(\mathcal{H}_{tj\ell}, \mathcal{H}_{ti\ell})
    = \sum_{j=1}^V a_{\ell j}^t Var(\mathcal{H}_{tj\ell})\leq (T-t)\sigma^2 + (T-t)(T-t-1)\rho.\nonumber
\end{align}
Here, we used the fact that we can allocate at most one token to position $\ell$, i.e., $\sum_{i=1}^V a_{\ell i}^t \leq 1$.
Now, observe that $\mathcal{H} = \sum_{t=0}^{T-1} \sum_{\ell=1}^L \mathcal{H}_{\ell t}$ is Gaussian with zero mean.
Remember, that since only a single token $j$ can be allocated to a specific position $\ell$ (Equation~(\ref{eq:feasibility})), then for any fixed feasible allocation $a_{\ell j}^t$ we must have:
\begin{align}
    \mathcal{H} = \sum_{\tau=0}^{T-1} \sum_{\ell=1}^L \sum_{j=1}^V b_{\ell j}^{\tau}(a) \eta_{\tau j \ell},
\end{align}
where $b_{\ell j}^{\tau}(a) = \sum_{t=1}^\tau a_{\ell j }^t$ indicates whether token $j$ is allocated to position $\ell$ by time step $\tau$ (in this case, all the following random contributions $\eta_{\tau j\ell}$ must have non-zero coefficients).
Noticing that $\sum_{\tau, j, \ell} b_{\ell j}^\tau(a) = \sum_{t\ell j} a_{\ell j}^t (T-t)\leq L(T-1)$ we can bound:
\begin{align}
    Var(\mathcal{H}) \leq LT\sigma^2 + L^2T^2\rho.\nonumber
\end{align}
Therefore, using the concentration inequality we can bound:
\begin{align}
    \Pr( \mathcal{H} \geq \delta) \leq \exp\Big\{ -\frac{\delta^2}{2 LT(\sigma^2 + LT\rho)} \Big\}.\nonumber
\end{align}

Notice that we want to bound $\mathcal{H}$ for any possible allocations.
As $\mathcal{H}$ is linear in $a_{\ell j}^t$ and we optimise over a convex feasible set, the supremum of $\mathcal{H}$ is attained at the deterministic allocations $a_{\ell j}^t$. 
Therefore, using the union bound we obtain: 
\begin{align}
    \Pr\Big(\sup_a \mathcal{H} \geq \delta\Big) \leq (VT+1)^L \exp\Big\{ -\frac{\delta^2}{2 LT(\sigma^2 + LT\rho)} \Big\},\nonumber
\end{align}
where $(VT+1)^L$ is the number of possible deterministic allocations.
Finally,
\begin{align}\label{eq:bound}
    \sum_{t=0}^{T-1}\sum_{\ell=1}^L\sum_{j=1}^V a_{\ell j}^t\log p_\theta (x_t^\ell = j | x_{t+1}) \leq &\sum_{t=0}^{T-1}\sum_{\ell=1}^L\sum_{j=1}^V a_{\ell j}^t\log p_\theta(x_T^\ell = j) + \\
    &\bar{\mu} LT + \sqrt{2LT(\sigma^2 + LT\rho)( \log\frac{1}{\delta} + L\log(VT+1))}.\nonumber
\end{align}
with probability at least $1-\delta$. 

From Equations (\ref{eq:dualized}) and (\ref{eq:bound}) we now obtain: 
\begin{align}\label{eq:dualized_simplified}
    \mathbb{E}_{a\sim a^*(\epsilon)} [Reward(a, \epsilon)\, |\, \epsilon\,] \leq & \min_{\lambda \geq 0} \max_{a_{\ell j}^t\in \Delta} \Big[\sum_{t=0}^{T-1} \sum_{\ell=1}^L \sum_{j=1}^V a_{\ell j}^{t} \Big(\log p_\theta(x_T^{\ell}=j) + \lambda(b_{\ell j} - \frac{R}{L})\Big) + H(a_{\ell}^t)\Big] \nonumber\\
    &  + \bar{\mu} LT + \sqrt{2LT(\sigma^2 + LT\rho)( \log\frac{1}{\delta} + L\log(VT+1))}
\end{align}
with probability at least $1-\delta$.
From the first-order conditions we obtain the optimal sampling probabilities:
\begin{align}\label{eq:sampling_probs}
    a_{\ell j}^* = \frac{p_\theta (x_T^\ell=j) e^{  \lambda(b_{\ell j} - R/L)}}{\sum_{i=1}^V p_\theta (x_T^\ell=i) e^{ \lambda(b_{\ell i} - R/L)}}. 
\end{align}
Importantly, as the coefficients $\log p_\theta(x_T^{\ell}=j)$ of the simplified dualised Problem (\ref{eq:dualized_simplified}) do not depend on $t$, the resulting solution in Equation~(\ref{eq:sampling_probs}) does not depend on $t$ as well. 
We now obtain: 
\begin{align}
    \mathbb{E}_{a\sim a^*(\epsilon)} [Reward(a, \epsilon)\, |\, \epsilon\,] \leq \min_{\lambda\geq 0}\; &T \cdot\sum_{\ell=1}^L \log\sum_{i=1}^V p_\theta (x_T^\ell=i) e^{\lambda(b_{\ell i} - R/L)}\nonumber\\
    &  + \bar{\mu} LT + \sqrt{2LT(\sigma^2 + LT\rho)( \log\frac{1}{\delta} + L\log(VT+1))}\nonumber
\end{align}
with probability at least $1-\delta$.\; Q.E.D.
\end{proof}

From Statement \ref{th:reward_OPT} it follows that in domains with the small (co-)variance values $\rho$, $\sigma$ and the small drift values $\bar{\mu}$ (see Section~\ref{app:text-temporal-consistency}) the approximation bounds of primal-dual inference may get significantly tighter. 

We now switch to deriving the lower bound on the objective value reached by Algorithm \ref{alg:pd}.

\begin{statement}\label{th:reward_alg}
    Expected reward of Algorithm \ref{alg:pd} is lower bounded by 
    \begin{align}
        \mathbb{E}_{a(\epsilon)}[Reward(\mathcal{A}, \epsilon) | \epsilon] \geq &\sum_{t=0}^{T-1} \sum_{\ell=1}^L\log \sum_{i=1}^V p_\theta (x_T^\ell=i ) e^{ \lambda_{t+1} (b_{\ell i} - R/L)}
    - \sum_{t=0}^{T-1}\sum_{\ell=1}^L \sum_{j=1}^V a_{\ell j}^t \lambda_{t+1} (b_{\ell j} - \frac{R}{L})\nonumber\\
    &-LT^2\bar{\mu}-LT\sqrt{2\log\frac{LVT}{\delta}T(\sigma^2 + (T-1)\rho)}
    \end{align}
    with probability at least $1-\delta$.
\end{statement}
\begin{proof}
    Consider the expected reward produced by Algorithm \ref{alg:pd} at iteration $t$:
    \begin{align}\label{eq:reward_iter}
        \mathbb{E}_{a(\epsilon)}[\pi_t | \epsilon] = \sum_{\ell=1}^L \sum_{j=1}^V a_{\ell j}^t \log p_\theta(x_t^\ell = j | x_{t+1}) + H(a_\ell^t),
    \end{align}
    where the sampling probability $a_{\ell j}^t$ is defined in Equation~(\ref{eq:sampling}).
    Substituting $a_{\ell j}^t$ in Equation~(\ref{eq:reward_iter}) we obtain: 
    \begin{align}\label{eq:reward_1}
        \mathbb{E}_{a(\epsilon)}[\pi_t | \epsilon] = \sum_{\ell=1}^L\log \sum_{i=1}^V p_\theta (x_t^\ell=i | x_{t+1}) e^{ \lambda_{t+1} (b_{\ell i} - R/L) }- \sum_{\ell=1}^L \sum_{j=1}^V a_{\ell j}^t \lambda_{t+1} (b_{\ell j} - \frac{R}{L}).
    \end{align}
We now rely on Assumption \ref{as:temporal_consistency} to simplify the expression above as follows:  
\begin{align}
    \mathbb{E}_{a(\epsilon)}[\pi_t| \epsilon] = \sum_{\ell=1}^L\log \sum_{i=1}^V p_\theta (x_T^\ell=i ) e^{ \lambda_{t+1} (b_{\ell i} - R/L)} e^{\sum_{\tau=t}^{T-1} \epsilon_{\tau i\ell}}
    - \sum_{\ell=1}^L \sum_{j=1}^V a_{\ell j}^t \lambda_{t+1} (b_{\ell j} - \frac{R}{L}).\nonumber
\end{align}
Observe, that for any real scalar value $X_{i\ell} > 0$ we have
\begin{align}
    \Pr\Big[\sum_{\tau=t}^{T-1} \epsilon_{\tau i\ell} \geq -X_{i\ell} \Big] \geq \Pr\Big[\mathcal{H}_{ti\ell} \geq (T-t)\bar{\mu} - X_{i\ell} \Big] =
    1 - \Pr\Big[\mathcal{H}_{ti\ell} < (T-t)\bar{\mu} - X_{i\ell} \Big].\nonumber
\end{align}
Now, if we let $X_{i\ell} = (T-t)\bar{\mu} + \gamma$, then we can use the concentration bound to obtain:
\begin{align}
    \Pr\Big[\sum_{\tau=t}^{T-1} \epsilon_{\tau i\ell} \geq -X_{i\ell} \Big] \geq 1 - \exp\Big\{-\frac{\gamma^2}{2 (T-t)(\sigma^2 + (T-t-1)\rho)}\Big\}.\nonumber
\end{align}
Consequently, using the union bound across $i$ and $\ell$ we can bound
\begin{align}
    \mathbb{E}_{a(\epsilon)}[\pi_t| \epsilon] \geq  &\sum_{\ell=1}^L\log \sum_{i=1}^V p_\theta (x_T^\ell=i ) e^{ \lambda_{t+1} (b_{\ell i} - R/L)}
    - \sum_{\ell=1}^L \sum_{j=1}^V a_{\ell j}^t \lambda_{t+1} (b_{\ell j} - \frac{R}{L})\nonumber\\
    &-L(T-t)\bar{\mu}-L\sqrt{2\log\frac{LV}{\delta}(T-t)(\sigma^2 + (T-t-1)\rho)}\nonumber
\end{align}
with probability at least $1-\delta$.
Finally, using the union bound across the time steps $t$ we obtain:
\begin{align}
    \mathbb{E}_{a(\epsilon)}\Big[\sum_{t=0}^{T-1} \pi_t\Big| \epsilon \Big] \geq &\sum_{t=0}^{T-1} \sum_{\ell=1}^L\log \sum_{i=1}^V p_\theta (x_T^\ell=i ) e^{ \lambda_{t+1} (b_{\ell i} - R/L)}
    - \sum_{t=0}^{T-1}\sum_{\ell=1}^L \sum_{j=1}^V a_{\ell j}^t \lambda_{t+1} (b_{\ell j} - \frac{R}{L})\nonumber\\
    &-LT^2\bar{\mu}-LT\sqrt{2\log\frac{LVT}{\delta}T(\sigma^2 + (T-1)\rho)}
\end{align}
with probability at least $1-\delta$.
Q.E.D.
\end{proof}

Now, once we derived the upper bound on the expected optimal reward (Statement \ref{th:reward_OPT}) and the lower bound on the expected reward reached by Algorithm \ref{alg:pd} (Statement \ref{th:reward_alg}), we can derive an upper bound on the regret of Algorithm \ref{alg:pd}.
First, we define the expected regret as
\begin{align}
    Regret_\epsilon(\mathcal{A}) = \mathbb{E}_{a\sim a^*(\epsilon)} [Reward(a, \epsilon)\, |\, \epsilon\,] - \mathbb{E}_{a(\epsilon)}[Reward(\mathcal{A}, \epsilon) |\, \epsilon\,]
\end{align}

We start with the following lemma:
\begin{lemma}\label{th:lemma}
    The regret of Algorithm \ref{alg:pd} is at most
    \begin{align}\label{eq:regret_lemma}
        Regret_\epsilon(\mathcal{A}) \leq &\sum_{t=0}^{T-1}\sum_{\ell=1}^L \sum_{j=1}^V a_{\ell j}^t \cdot \lambda_{t+1} (b_{\ell j} - \frac{R}{L}) + \bar{\mu} LT(T+1) + \\
        &\sqrt{4(\sigma^2 + LT\rho) \log\frac{2LVT(VT+1)^L}{\delta} (LT+L^2T^3)}\nonumber
    \end{align}
    with probability at least $(1-\delta)$. 
\end{lemma}
\begin{proof}
    Let 
    \begin{align}
        \lambda^* \in\arg\min_\lambda \sum_{\ell=1}^L \log\sum_{i=1}^V p_\theta (x_T^\ell=i) e^{ \lambda(b_{\ell i} - R/L)},\nonumber
    \end{align}
    see Equation~(\ref{th:opt_bound}).
    This implies that 
    \begin{align}
        \sum_{\ell=1}^L \log\sum_{i=1}^V p_\theta (x_T^\ell=i) e^{ \lambda^*(b_{\ell i} - R/L)} \leq \sum_{\ell=1}^L \log\sum_{i=1}^V p_\theta (x_T^\ell=i) e^{ \lambda_t(b_{\ell i} - R/L)},  \;\;\forall t=1,...,T.\nonumber
    \end{align}
    Consequently,
    \begin{align}
        T\sum_{\ell=1}^L \log\sum_{i=1}^V p_\theta (x_T^\ell=i) e^{ \lambda^*(b_{\ell i} - R/L)} \leq \sum_{t=0}^{T-1}\sum_{\ell=1}^L \log\sum_{i=1}^V p_\theta (x_T^\ell=i) e^{ \lambda_t(b_{\ell i} - R/L)}.\nonumber
    \end{align}
    Therefore, from Statements \ref{th:reward_OPT} and \ref{th:reward_alg} it follows:
    \begin{align}
        \mathbb{E}_{a\sim a^*(\epsilon)}&[Reward(a, \epsilon)\, |\, \epsilon\,] \leq \mathbb{E}_{a(\epsilon)}[Reward(\mathcal{A}, \epsilon) |\, \epsilon\,] + \sum_{t=0}^{T-1}\sum_{\ell=1}^L \sum_{j=1}^V a_{\ell j}^t \cdot \lambda_{t+1} (b_{\ell j} - \frac{R}{L}) \nonumber\\
        + &\bar{\mu} LT(T+1) + \sqrt{4(\sigma^2 + LT\rho) \log\frac{2LVT(VT+1)^L}{\delta} (LT+L^2T^3)} \nonumber
    \end{align}    
with probability at least $1-\delta$. Q.E.D.
\end{proof}

We are now ready to prove our final regret bound:
\OnlineAlgRegret*
\begin{proof}
    Let $\delta_t$ be a per-step reward in Equation~(\ref{eq:mirror}), i.e.,
    \begin{align}
        \delta_t = \sum_{\ell=1}^L \sum_{j=1}^V (b_{\ell j} - \frac{R}{L})a_{\ell j}^t,
    \end{align}
    that is defined for any deterministic allocation $a_{\ell j}^t$ at time step $t$ of the reverse diffusion process.
    In this case, we can re-write Equation~(\ref{eq:mirror}) as follows:
    \begin{align}
        \lambda_t = \lambda_{T} \exp\Big\{ -\eta \sum_{t=t}^{T-1}\delta_t \Big\}.
    \end{align}
    Now, we can bound the first sum in Equation~(\ref{eq:regret_lemma}) as follows: 
    \begin{align}
        &\sum_{t=0}^{T-1}\sum_{\ell=1}^L \sum_{j=1}^V a_{\ell j}^t \cdot \lambda_{t+1} (b_{\ell j} - \frac{R}{L}) =\sum_{t=0}^{T-1} \lambda_{t+1} \delta_t \leq \sum_{t=0}^{T-1} \lambda_{t+1} |\delta_t|.
    \end{align}
    Let $B = L \cdot \max_{\ell,j} |b_{\ell j} - \frac{R}{L}|$, so that $|\delta_{t}|\leq B$ for all $t=0,...,T-1$.
    Consequently, we can bound:
    \begin{align}
        \sum_{t=0}^{T-1} \lambda_{t+1} |\delta_t| \leq B \lambda_{T}\sum_{t=0}^{T-1}  \exp\{\eta B (T-t-1)\} \leq B\lambda_T \frac{e^{\eta B T} - 1}{e^{\eta B} - 1} \leq \frac{\lambda_T}{\eta} e^{\eta B T}.
    \end{align}
\end{proof}

\subsection{Bounds on Constraint Violation}\label{app:cv_bounds}

We now illustrate how to bound the expected value of constraint violations reached by Algorithm \ref{alg:pd}.

\OnlineAlgCV*
\begin{proof}
We let $(\Pi_{min}, \Pi_{max})$ be the support of the log probabilities distribution, and let $b_{max}$ be the maximal contribution of the tokens towards the constraint, i.e.,
\begin{align}
    b_{max} = \max_{\ell, j} b_{\ell j}.
\end{align}
We let $r_t$ be the \textit{marginal reward} obtained by Algorithm \ref{alg:pd} at iteration $t$, i.e.,
\begin{align}
    r_t = \sum_{\ell=1}^L \Big(b_{\ell j(t)} - \frac{R}{L}\Big),
\end{align}
where $j(t)$ is the index of the token allocated to position $\ell$ at step $t$.
Remember, that the expected number of masked tokens at time $t+1$ of the reverse diffusion process can be computed as $L (1-\alpha_{t+1})$, and the probability that the token gets unmasked at step $t$ is $\frac{\alpha_t - \alpha_{t+1}}{1-\alpha_{t+1}}$. 
Therefore, the trivial lower bound on the expected marginal reward $r_t$ at time $t$ is
\begin{align}
    \mathbb{E}_{p_{\theta}}[r_t] \geq -\frac{R}{L}\cdot \underbrace{\frac{\alpha_t - \alpha_{t+1}}{1-\alpha_{t+1}} \cdot L (1-\alpha_{t+1})}_{\substack{\text{Expected number of tokens} \\ \text{unmasked at time $t$}}} = -R \cdot (\alpha_t - \alpha_{t+1}).
\end{align}

Now, let us choose some $\beta\in(0,b_{max})$, and assume that for all steps $t\geq \tilde{t}$ we have $\lambda_t \geq (\eta+1)\cdot \frac{\Pi_{max}- \Pi_{min}}{b_{max}-\beta}$. 
Then, for any tokens $i$ and $j$ s.t., $b_{max} = b_{\ell j} > \beta \geq b_{\ell i} $ we have:
\begin{align}
    \underbrace{\big(\log p_\theta(j) + \lambda_t b_{\ell j}\big)}_{s_{\ell j}} - &\underbrace{\big(\log p_\theta(i) + \lambda_t b_{\ell i}\big)}_{s_{\ell i}} =
    \log p_\theta(j) - \log p_\theta(i) + \lambda_t \big(b_{\ell j} - b_{\ell i}\big) \geq\nonumber \\
    &-(\Pi_{max} - \Pi_{min}) + (\eta+1)\cdot\frac{\Pi_{max}- \Pi_{min}}{b_{max}-\beta} \big( b_{max} - b_{\ell i} \big)\geq\nonumber\\
    &(\Pi_{max} - \Pi_{min})\eta\Big[ \frac{b_{max}-\beta}{b_{max}-\beta}\Big] \geq \eta (\Pi_{max} - \Pi_{min}).
\end{align}

This means that for any time step $t=1,...,\tilde{t}$ the score $s_{\ell j}$ of the token with the largest attribute $b_{max}$ is at least $\eta(\Pi_{max}-\Pi_{min})$ larger that the score of any token with attribute $b_{\ell i} \leq \beta$, i.e., $s_{\ell i} \leq s_{\ell j} - \eta(\Pi_{max} - \Pi_{min})$.
Therefore, the probability of sampling tokens with attributes at least $\beta$ can be lower bounded as follows: 
\begin{align}
    \frac{\exp\{s_{\ell j}\}}{ \sum_{i=1}^V \exp\{s_{\ell i}\}} \geq \frac{e^{s_{\ell j}}}{ \exp\{s_{\ell j}\} + (V-1) e^{s_{\ell j} - \eta(\Pi_{max} - \Pi_{min})}} = 
    \frac{1}{1+(V-1)e^{-\eta(\Pi_{max} - \Pi_{min})}}.
\end{align}

Consequently, we can lower bound the expected marginal reward for any time steps 1,...,$\tilde{t}$: 
\begin{align}
    \mathbb{E}_{p_\theta}[r_t] \geq  \frac{(L \beta - R)(\alpha_t - \alpha_{t+1}) }{1+(V-1)\exp\{-\eta(\Pi_{max} - \Pi_{min})\}}\;\;\;\;\;\forall t=1,...,\tilde{t}.
\end{align}
Observe, that by choosing sufficiently large value of $\eta$, we can make this marginal reward being very close to its maximal possible value $b_{max}-\frac{R}{L}$ for each unmasked token. 

Now, the total constraint violation:
\begin{align}
    -\sum_{t=1}^T \mathbb{E}_{p_\theta}[r_t] =& -\sum_{t=\tilde{t}+1}^T \mathbb{E}_{p_\theta}[r_t] - \sum_{t=1}^{\tilde{t}} \mathbb{E}_{p_\theta}[r_t] \leq\nonumber\\
    & R(\alpha_{\tilde{t}+1} - \alpha_{T}) - \frac{(L \beta - R)(\alpha_1 - \alpha_{\tilde{t}+1}) }{1+(V-1)\exp\{-\eta(\Pi_{max} - \Pi_{min})\}} \leq\nonumber \\
    &R(\alpha_{\tilde{t}+1} - \alpha_{T}) - (L \beta - R)(\alpha_1 - \alpha_{\tilde{t}+1}) \leq\nonumber \\
    & R(\alpha_1 -\alpha_T)  - L\beta (\alpha_1 - \alpha_{\tilde{t}+1}).\nonumber
\end{align} 

Notice, that for Algorithm \ref{alg:pd} to reach $\tilde{\lambda} = (\eta+1)\cdot \frac{\Pi_{max}-\Pi_{min}}{b_{max}-\beta}$, it must accumulate the total expected marginal reward of at least $-R(\alpha_{\tilde{t}} - \alpha_T)$ at which point 
\begin{align}
\alpha_{\tilde{t}} = \frac{\log \tilde{\lambda}}{\eta R}.
\end{align}
Consequently, 
\begin{align}
    \alpha_{\tilde{t}+1} \leq \alpha_{\tilde{t}} = \frac{1}{\eta R} \log \Big\{(\eta+1)\cdot \frac{\Pi_{max}-\Pi_{min}}{b_{max}-\beta}\Big\}.
\end{align}
Thus, the total constraint violation must be bounded by  
\begin{align}
    -\sum_{t=1}^T \mathbb{E}_{p_\theta}[r_t]\leq R - L\beta (1 - \frac{1}{\eta R} \log \Big\{(\eta+1)\cdot \frac{\Pi_{max}-\Pi_{min}}{b_{max}-\beta}\Big\}). 
\end{align}
We set $\beta = R/L$ to obtain:
\begin{align}
    -\sum_{t=1}^T \mathbb{E}_{p_\theta}[r_t]\leq \frac{1}{\eta} \log\Big\{ (\eta+1)\cdot \frac{\Pi_{max}-\Pi_{min}}{b_{max}-\frac{R}{L}} \Big\}.
\end{align}
Q.E.D.
\end{proof}

\section{Extended Results}\label{sec:extended-results}

\subsection{Text Generation Details}\label{app:text-details}

\begin{table}[htbp]
\centering
\caption{TinyStories ocean steering, $N{=}1{,}000$. SPDD: accumulated slack, $\eta{=}2.0$. LLaDA-8B transfer in Appendix~\ref{app:llada}.}
\label{tab:ocean}
\begin{tabular}{lccccc}
\toprule
Method & Target (avg)  & Pass\% & PPL$\downarrow$ & Dist-2$\uparrow$ & KL$\downarrow$ \\
\midrule
Unconstrained & 0.3 & \phantom{0}0.5 $\pm$ 0.2\% & 24.7 & 0.187 & 0.000 \\
Static bias ($\alpha\!=\!2.0$) & 19.2 & 81.9 $\pm$ 1.2\% & 24.8 & 0.169 & 0.706 \\
CDD ($\tau{=}0.5$) & \phantom{0}8.1 & 38.0 $\pm$ 1.5\% & 24.6 & 0.318 & 0.221 \\
D-CBG ($\gamma{=}5$) & 0.1 & \phantom{0}0.0 $\pm$ 0.0\% & 25.2 & 0.336 & 0.127 \\
GPT-4.1-mini & 31.5 & 100\phantom{$\pm$00.0}\% & 27.1 & 0.082 & 2.852 \\
\midrule
SPDD (Ours) & 11.0 & 92.1 $\pm$ 0.9\% & 25.3 & 0.175 & 0.420 \\
\bottomrule
\end{tabular}
\end{table}

We organise this appendix into three groups: (i) hyperparameter sweeps that establish the operating-point robustness claimed in Section~\ref{sec:experiments}, (ii) qualitative illustrations of how SPDD shapes generation step-by-step, and (iii) empirical verification of the temporal-consistency assumption that underpins Statement~\ref{th:regret}.

\paragraph{Slack Mode Ablation.}\label{app:text-slack}

\begin{table}[htbp]
\centering
\caption{Slack mode ablation, per-position $\lambda$, $R{=}10$, $N{=}1{,}000$ (same setup as Table~\ref{tab:ocean}). Ocean reported as mean $\pm$ per-sample SE; Ocean$/R$ is the overshoot ratio.}
\label{tab:ocean-slack}
\begin{tabular}{llrrrrrr}
\toprule
Slack mode & $\eta$ & Ocean & Ocean$/R$ & Pass\% & PPL$\downarrow$ & Dist-2$\uparrow$ & KL$\downarrow$ \\
\midrule
Accumulated & 0.1 & \phantom{0}6.5 $\pm$ 4.0 & 0.65 & 20.5 $\pm$ 1.3\% & 25.5 & 0.185 & \textbf{0.281} \\
Accumulated & 0.5 & \phantom{0}9.5 $\pm$ 2.6 & 0.95 & 46.9 $\pm$ 1.6\% & 25.6 & 0.181 & \textbf{0.365} \\
Accumulated & 1.0 & 10.5 $\pm$ 2.0 & 1.05 & 67.6 $\pm$ 1.5\% & 25.7 & 0.180 & \textbf{0.402} \\
Accumulated & 2.0 & 11.0 $\pm$ 1.7 & 1.10 & 92.1 $\pm$ 0.9\% & 25.3 & 0.175 & \textbf{0.420} \\
\midrule
Instantaneous & 0.1 & 10.9 $\pm$ 5.5 & 1.09 & 55.2 $\pm$ 1.6\% & 24.8 & 0.178 & 0.412 \\
Instantaneous & 0.5 & 21.4 $\pm$ 4.4 & 2.14 & 100\phantom{$\pm$00.0}\% & \textbf{23.0} & 0.159 & 0.890 \\
Instantaneous & 1.0 & 21.5 $\pm$ 4.3 & 2.15 & 100\phantom{$\pm$00.0}\% & 24.5 & 0.161 & 0.916 \\
Instantaneous & 2.0 & 21.8 $\pm$ 4.8 & 2.18 & 100\phantom{$\pm$00.0}\% & 24.8 & 0.164 & 0.924 \\
\midrule
Early & 0.1 & \phantom{0}8.5 $\pm$ 5.5 & 0.85 & 37.0 $\pm$ 1.5\% & 25.2 & 0.182 & 0.342 \\
Early & 0.5 & 19.6 $\pm$ 4.4 & 1.96 & 99.6 $\pm$ 0.2\% & 23.7 & 0.160 & 0.828 \\
Early & 1.0 & 20.5 $\pm$ 4.5 & 2.05 & 100\phantom{$\pm$00.0}\% & 24.0 & 0.161 & 0.879 \\
Early & 2.0 & 21.0 $\pm$ 4.6 & 2.10 & 100\phantom{$\pm$00.0}\% & 24.6 & 0.164 & 0.892 \\
\midrule
Optimistic (OMD) & 0.1 & 10.6 $\pm$ 5.4 & 1.06 & 53.1 $\pm$ 1.6\% & 24.8 & 0.176 & 0.399 \\
Optimistic (OMD) & 0.5 & 21.3 $\pm$ 4.3 & 2.13 & 100\phantom{$\pm$00.0}\% & 23.2 & 0.158 & 0.898 \\
Optimistic (OMD) & 1.0 & 21.4 $\pm$ 4.6 & 2.14 & 100\phantom{$\pm$00.0}\% & 23.7 & 0.161 & 0.930 \\
Optimistic (OMD) & 2.0 & 21.5 $\pm$ 4.6 & 2.15 & 100\phantom{$\pm$00.0}\% & 24.7 & 0.166 & 0.919 \\
\bottomrule
\end{tabular}
\end{table}

Slack computation strongly affects both satisfaction and how far the constrained distribution drifts from the unconstrained model. \emph{Accumulated} stays closest to the target ($\bar{c}\!\approx\!R$) and is uniformly the lowest-KL mode at every $\eta$ (KL\,$\leq$\,0.42), at the cost of slower satisfaction ramp (21\%\,$\to$\,92\% across $\eta\!\in\![0.1,2]$) and elevated PPL. \emph{Instantaneous}, \emph{early}, and \emph{optimistic} all saturate pass rate at $\eta\!\geq\!0.5$ but at $2{-}3\times$ higher KL ($0.83{-}0.93$): $\lambda$ saturates at the clamp ceiling from step~0 (bang-bang control~\citep{bryson1975applied}), irreversibly committing early tokens and pulling the output distribution toward a thin slice of ocean-heavy text. Within those three saturating modes, instantaneous achieves the lowest PPL at $\eta\!=\!0.5$ (23.0), with optimistic and early within $\pm0.5$ PPL. Optimistic's distinguishing benefit is at the moderate $\eta\!=\!0.1$ regime, where its prediction correction lifts pass rate from 49\% (instantaneous) and 37\% (early) to 53\% at slightly lower KL (0.40 vs.\ 0.41 / 0.34), the regime where the constraint is binding but the model's own forecasts are still informative.

\paragraph{Optimistic Mirror Descent Slack Mode.}

Standard mirror descent updates $\lambda$ using only the current constraint violation $h_k = R - c_k$, where $c_k$ is the current count and $R$ is the target. Optimistic Mirror Descent (OMD)~\citep{rakhlin2013online} improves convergence by incorporating a prediction $M_k$ of the \emph{next} violation, yielding the combined update:
\begin{equation}
    h^\text{opt}_k = h_k + \tilde{M}_k - \tilde{M}_{k-1}, \quad
    \lambda_{k+1} = \lambda_0 \cdot \exp\!\big(\eta \cdot h^\text{opt}_k\big)
\end{equation}
where $\tilde{M}_k = (1 - \bar{m}_k) \cdot M_k$ anneals the prediction by the unmasked fraction $\bar{m}_k$, and $M_k = R - (c_k + \hat{E}_k)$ with $\hat{E}_k = \sum_{\ell \in \text{masked}} \sum_j p_\theta(j \mid x_k) \cdot s_j$ being the expected future contribution from masked positions.

This is a natural fit for discrete diffusion: early in denoising, most tokens are masked and predictions are unreliable (annealed to zero, reducing to standard mirror descent); late in denoising, few tokens remain masked and predictions are highly accurate, giving near-optimal lambda adjustment. The OMD regret bound replaces $\sum_k \|g_k\|^2$ with $\sum_k \|g_k - M_k\|^2$, which is small precisely when it matters most (late steps).

Table~\ref{tab:ocean-slack} reports the OMD row alongside accumulated/instantaneous/early slack at matched $\eta$. At $\eta{=}0.5$, OMD achieves 98\% pass rate with PPL 23.0, the best fluency across all slack modes at high constraint satisfaction. The prediction correction smooths the $\lambda$ trajectory, reducing over-guidance that degrades fluency. At $\eta{=}1.0$, OMD matches instantaneous (100\% pass) with comparable metrics. The computational overhead is negligible (${\sim}1$\% additional time for the softmax prediction at each step). OMD is most beneficial at moderate $\eta$ where the standard modes must choose between satisfaction and fluency; at aggressive $\eta$, all modes saturate the constraint.

\paragraph{Constraint-Difficulty ($R$) Sweep.}

\begin{table}[htbp]
\centering
\caption{Constraint-difficulty sweep, per-position $\lambda$, $\eta\!=\!0.5$, $N\!=\!1{,}000$, $R\!\in\!\{5,10,15,20\}$. Ocean is mean $\pm$ per-sample SE. Complements Figure~\ref{fig:rsweep}.}
\label{tab:ocean-rsweep}
\begin{tabular}{llrrrrr}
\toprule
$R$ & Slack mode & Ocean & Pass\% & PPL$\downarrow$ & Dist-2$\uparrow$ & KL$\downarrow$ \\
\midrule
\phantom{0}5 & Optimistic    & 13.7 $\pm$ 4.3 & 83.1 $\pm$ 1.2\% & 23.9 & 0.166 & 0.558 \\
\phantom{0}5 & Instantaneous & 13.9 $\pm$ 4.4 & 83.7 $\pm$ 1.2\% & 23.9 & 0.165 & 0.583 \\
\phantom{0}5 & Accumulated   & \phantom{0}6.3 $\pm$ 2.8 & 11.5 $\pm$ 1.0\% & 25.8 & 0.187 & \textbf{0.260} \\
\midrule
10 & Optimistic    & 21.3 $\pm$ 4.4 & 99.9 $\pm$ 0.1\% & 23.8 & 0.163 & 0.906 \\
10 & Instantaneous & 21.2 $\pm$ 4.4 & 100\phantom{$\pm$00.0}\% & 23.4 & 0.160 & 0.903 \\
10 & Accumulated   & \phantom{0}9.6 $\pm$ 2.6 & 45.1 $\pm$ 1.6\% & 25.6 & 0.179 & \textbf{0.372} \\
\midrule
15 & Optimistic    & 26.9 $\pm$ 4.7 & 100\phantom{$\pm$00.0}\% & 23.4 & 0.155 & 1.111 \\
15 & Instantaneous & 27.5 $\pm$ 4.8 & 100\phantom{$\pm$00.0}\% & \textbf{23.2} & 0.156 & 1.135 \\
15 & Accumulated   & 13.1 $\pm$ 2.5 & 93.1 $\pm$ 0.8\% & 25.6 & 0.176 & \textbf{0.475} \\
\midrule
20 & Optimistic    & 32.7 $\pm$ 4.7 & 100\phantom{$\pm$00.0}\% & 23.9 & 0.155 & 1.354 \\
20 & Instantaneous & 32.9 $\pm$ 4.8 & 100\phantom{$\pm$00.0}\% & 23.4 & 0.152 & 1.340 \\
20 & Accumulated   & 16.9 $\pm$ 2.4 & 99.5 $\pm$ 0.2\% & 25.8 & 0.172 & \textbf{0.625} \\
\bottomrule
\end{tabular}
\end{table}

\begin{figure}[htbp]
\centering
\includegraphics[width=\linewidth]{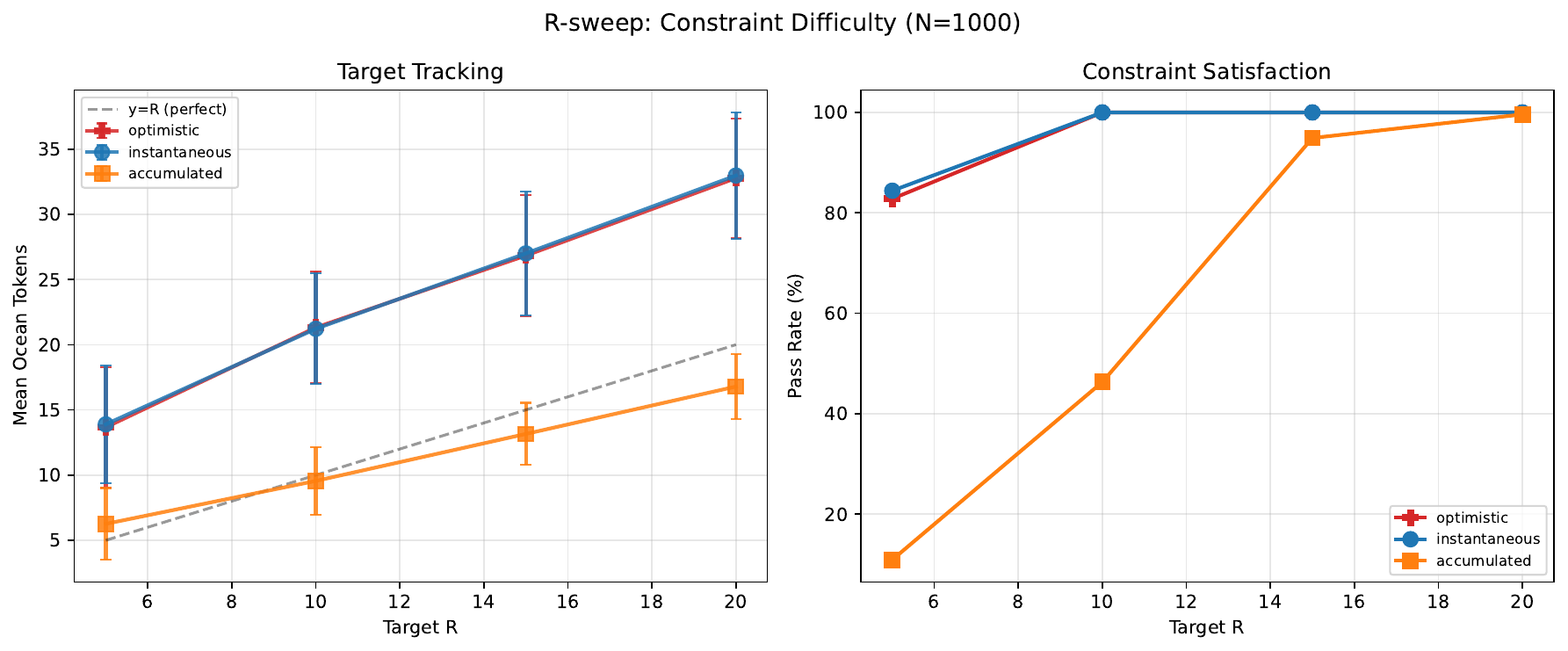}
\caption{Constraint difficulty sweep ($R \in \{5, 10, 15, 20\}$, $\eta\!=\!0.5$, $N\!=\!1{,}000$). Accumulated slack stays closest to the target and lowest-KL; optimistic and instantaneous saturate pass rate from $R\!=\!10$ at higher KL.}
\label{fig:rsweep}
\end{figure}

The sweep exposes three effects the main-text figure alone cannot show. First, \emph{accumulated lags asymmetrically with $R$}: at $R\!=\!5$ the prorated per-step target is too small to drive $\lambda$ off its floor and the method barely fires (11.5\% pass), whereas at $R\!=\!20$ the target is large enough that even the conservative ramp satisfies the constraint (99.5\% pass, $\bar{c}\!=\!16.9$). Instantaneous and optimistic show no such regime dependence. Second, \emph{accumulated dominates on KL at every $R$} (0.26--0.63 vs.\ 0.56--1.35 for instantaneous/optimistic), making it the right mode when distributional fidelity matters more than peak satisfaction. Third, \emph{KL grows roughly linearly with $R$} for the saturating modes (instantaneous: $0.58 \to 1.34$ across $R\!\in\![5,20]$), while accumulated KL grows much more slowly ($0.26 \to 0.63$): the constraint deficit drives both pass rate and KL, and accumulated's progress-prorated target absorbs the deficit instead of the model's distribution.

\paragraph{KL--pass rate Pareto.}

Figure~\ref{fig:pareto} visualises the cross-method trade-off on text from Table~\ref{tab:unified}: KL divergence against the unconstrained distribution on the $x$-axis, Pass\% on the $y$-axis. SPDD attains the highest Pass\% (92.1\%) at substantially lower KL than the next-best high-Pass baseline (Static at 81.9\% Pass / KL $0.71$); CDD trades much lower Pass for moderately lower KL; D-CBG cannot satisfy the sparse 18-token constraint at any tested $\gamma$. The dashed line shows the empirical Pareto frontier across these methods.

\begin{figure}[htbp]
\centering
\includegraphics[width=0.7\linewidth]{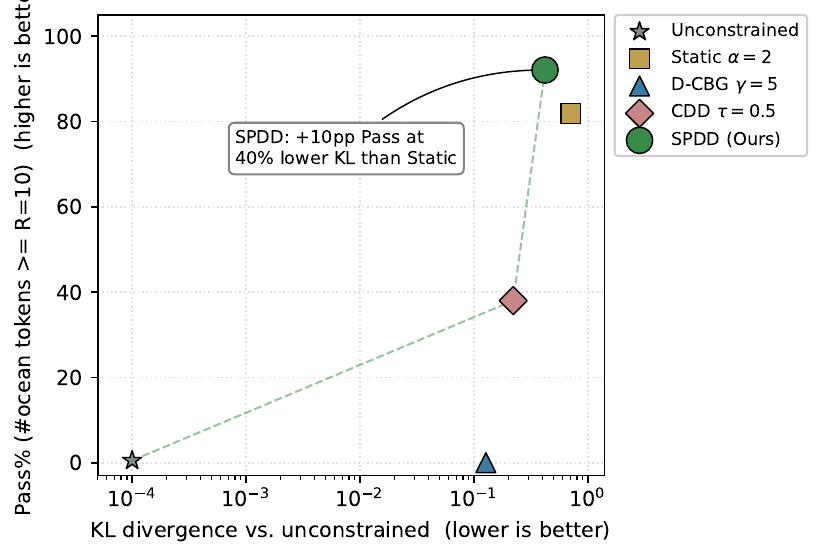}
\caption{Text Pareto front: unigram KL vs.\ unconstrained against Pass\% (\#ocean tokens~$\geq R{=}10$), at the operating points reported in Table~\ref{tab:unified} ($N{=}1{,}000$, 142M-MDLM TinyStories backbone). SPDD reaches the upper-right corner of the achievable region: highest Pass\%, with strictly lower KL than Static bias. The dashed line traces the empirical Pareto frontier (Unconstrained $\to$ CDD $\to$ SPDD).}
\label{fig:pareto}
\end{figure}

Beyond aggregate metrics, the next two paragraphs illustrate qualitatively how SPDD shapes generation, both as a step-by-step denoising trajectory and as completed sample stories.

\paragraph{Denoising Progression by Slack Mode.}

Table~\ref{tab:denoising-progression} shows the denoising progression for the TinyStories ocean-guided generation experiment in the main text. For each of three slack modes (accumulated, instantaneous, optimistic), we display the partial token sequence at four stages of denoising (25\%, 50\%, 75\%, 100\% of $K\!=\!256$ steps), using per-position $\lambda$ with $\eta\!=\!0.5$ and target $R\!=\!10$. Underscores denote masked (uncommitted) positions; ocean words are in \textbf{bold}.

The timing of ocean token placement differs substantially across slack modes. With instantaneous slack, 11 of 21 final ocean tokens are already committed at 25\% of denoising: $\lambda$ saturates at the clamp ceiling from step~0, forcing every early token to be ocean-themed. With optimistic slack, ocean tokens appear more gradually (7 at 25\%, 10 at 50\%), because the prediction correction $\hat{E}_k$ tempers the guidance as context builds. Accumulated slack is the slowest, with only 4 ocean tokens at 25\%, reflecting its conservative target prorating.

\begin{table}[htbp]
\centering
\small
\caption{Denoising progression for a single sample ($\eta\!=\!0.5$, $R\!=\!10$). First ${\sim}40$ tokens shown at each stage. Underscores are masked positions; ocean words in \textbf{bold}.}
\label{tab:denoising-progression}
\begin{tabular}{llcp{10cm}}
\toprule
Mode & Stage & Ocean & Token sequence (first ${\sim}40$ tokens) \\
\midrule
\multirow{4}{*}{\rotatebox{90}{Accumulated}}
& 25\% & 4 & \_ \_ \_ \_ \_ \_ the \textbf{beach} near \_ \_ \_ \_ \_ a big \_ with \_ \_ \_ \_ \_ splash \_ the \_ . `` \_ ' \_ \_ \_ \textbf{wave} \_ the \_ \_ \_ \_ \_ \_ \\
& 50\% & 9 & \_ and \_ \_ \_ \_ the \textbf{beach} near the \textbf{ocean} \_ They \_ a big \_ with some seagulls \_ They liked \_ \_ the waves splash \_ the waves . `` Let ' s \_ \_ \textbf{wave} to the \textbf{shore} \\
& 75\% & 12 & \_ and Mia \_ \_ \_ the \textbf{beach} near the \textbf{ocean} . They had a big truck with some seagulls . They liked to watch the waves splash \_ the waves . `` Let ' s \_ \_ \textbf{wave} to the \textbf{shore} \\
& 100\% & 15 & Tom and Mia were playing on the \textbf{beach} near the \textbf{ocean} . They had a big truck with some seagulls . They liked to watch the waves splash and the waves . `` Let ' s send a \textbf{wave} to the \textbf{shore} \\
\midrule
\multirow{4}{*}{\rotatebox{90}{Instantaneous}}
& 25\% & 11 & One \_ \_ \_ and \_ \_ to \_ \_ with \_ mom \_ dad \_ They \_ \_ boat \_ \_ \textbf{dolphin} jumped out \_ \_ \textbf{dolphin} was \_ \_ of \_ \_ \_ \_ \_ \_ \textbf{dolphin} \_ closer \\
& 50\% & 15 & One \_ , Ben and \_ went to \_ \_ with \_ mom \_ dad \_ They saw a big boat \_ \_ \textbf{dolphin} jumped out \_ \_ \textbf{dolphin} was \_ \_ of \_ \_ \_ . Ben \_ \_ \textbf{dolphin} \_ closer \\
& 75\% & 19 & One \_ , Ben and Lily went to the \textbf{beach} with their mom \_ dad \_ They saw a big boat and a \textbf{dolphin} jumped out . \_ \textbf{dolphin} was a \_ of a \_ tail . Ben \_ the \textbf{dolphin} \_ closer \\
& 100\% & 21 & One day , Ben and Lily went to the \textbf{beach} with their mom and dad . They saw a big boat and a \textbf{dolphin} jumped out . The \textbf{dolphin} was a \textbf{fish} of a shiny tail . Ben and the \textbf{dolphin} came closer \\
\midrule
\multirow{4}{*}{\rotatebox{90}{Optimistic}}
& 25\% & 7 & \_ \_ and \_ \_ to go \_ \_ \textbf{beach} \_ \_ \_ \_ \_ slow \_ \_ \_ loved \_ play \_ \_ \_ \_ that \_ \_ \_ \_ in \_ \_ \_ \_ \_ down \_ the \_ \_ \_ \\
& 50\% & 10 & \_ \_ and Mia \_ to go \_ \_ \textbf{beach} with their \_ \_ \_ slow turtle \_ \_ loved to play \_ \_ \_ \_ that they kicked \_ \_ in \_ \_ \_ \_ \_ down \_ the \textbf{shore} \\
& 75\% & 15 & Sam and Mia liked to go \_ the \textbf{beach} with their mom \_ \_ slow turtle . \_ loved to play with the \_ \textbf{whale} that they kicked and \_ in \_ \_ \_ \_ came down \_ the \textbf{shore} \\
& 100\% & 18 & Sam and Mia liked to go to the \textbf{beach} with their mom and the slow turtle . They loved to play with the big \textbf{whale} that they kicked and splashed in . But the \textbf{whale} came down from the \textbf{shore} \\
\bottomrule
\end{tabular}
\end{table}

\paragraph{Qualitative Examples.}

\begin{table}[htbp]
\centering
\small
\caption{Generated story excerpts (truncated to ${\sim}80$ words). Ocean words in \textbf{bold}.}
\label{tab:examples}
\begin{tabular}{p{2.2cm}p{10.3cm}}
\toprule
Method & Story excerpt \\
\midrule
Unconstrained \newline (0 ocean) & Lily and Ben were playing in the backyard. They had a big bowl of water. They wanted to make yummy cakes and cookies. ``Can you mix too, Lily?'' Ben asked his sister. ``Yes, yes, yes!'' Lily said. She helped him mix some flowers and mud shapes. She made a tasty cake also with two sprinkles\ldots \\
\midrule
Optimistic \newline $\eta\!=\!0.5$ \newline (10 ocean) & Tim and Lily are friends who like to play at the \textbf{beach}. They have a ball and a net. They like to kick the ball and chase each other. One day, they see a big dog at the \textbf{beach}. It looks mighty. Tim and Lily think it is friendly to play. They run to the dog and try to grab the ball. The dog sees Tim and Lily are afraid and \textbf{wave} their ball\ldots \\
\midrule
Instantaneous \newline $\eta\!=\!1.0$ \newline (10 ocean) & Lily and Ben liked to play at the \textbf{beach}. They liked to kick the ball and dig holes and look for shells. One day, they saw a \textbf{dolphin} in the water. ``Look, look!'' Lily said, watching the \textbf{dolphin}. ``It is so pretty and shiny.'' ``Go away, Lily!'' Ben said. ``No, \textbf{shark}s are mean and rude''\ldots \\
\bottomrule
\end{tabular}
\end{table}

Statement~\ref{th:regret} rests on Assumption~\ref{as:temporal_consistency}: that logit increments factor as a slow drift plus near-zero-covariance Gaussian residuals. The remaining paragraphs verify this empirically on text and molecular MDLMs, and then check whether the downstream bound holds in practice.

\paragraph{Empirical Verification of Weak Temporal Consistency.}\label{app:text-temporal-consistency}

Assumption~\ref{as:temporal_consistency} requires: (i) the logit increments $\varepsilon_{tj\ell}$ are Gaussian with bounded variance $\sigma^2$, and (ii) the cross-position covariance $\mathrm{Cov}(\varepsilon_{tj\ell}, \varepsilon_{tjk})$ is at most a small constant $\rho$. We test these empirically on the 142M MDLM trained on TinyStories ($N{=}10$ samples, $L{=}256$, $T{=}128$ denoising steps). Increments are computed only at positions that remain masked at both consecutive steps (once unmasked, logits are frozen and $\varepsilon$ trivially vanishes). After decomposing $\varepsilon_{tj\ell} = \mu_{tj\ell} + \eta_{tj\ell}$ via the per-($t,j,\ell$) sample mean $\hat\mu_{tj\ell}$, we inspect the residual $\eta$.

Figure~\ref{fig:temporal-consistency} (top) shows the raw $\varepsilon$ distribution and Q--Q plot; Figure~\ref{fig:temporal-consistency} (bottom) shows the centred residual $\eta$. The centred distribution is visibly closer to Gaussian in the bulk, confirming that a large fraction of the deviation from Gaussianity is absorbed by the drift term $\mu_{tj\ell}$. A heavy left tail remains at large unmasking events; this is consistent with the assumption holding in expectation but admitting occasional outliers.

\begin{figure}[h]
\centering
\includegraphics[width=0.95\linewidth]{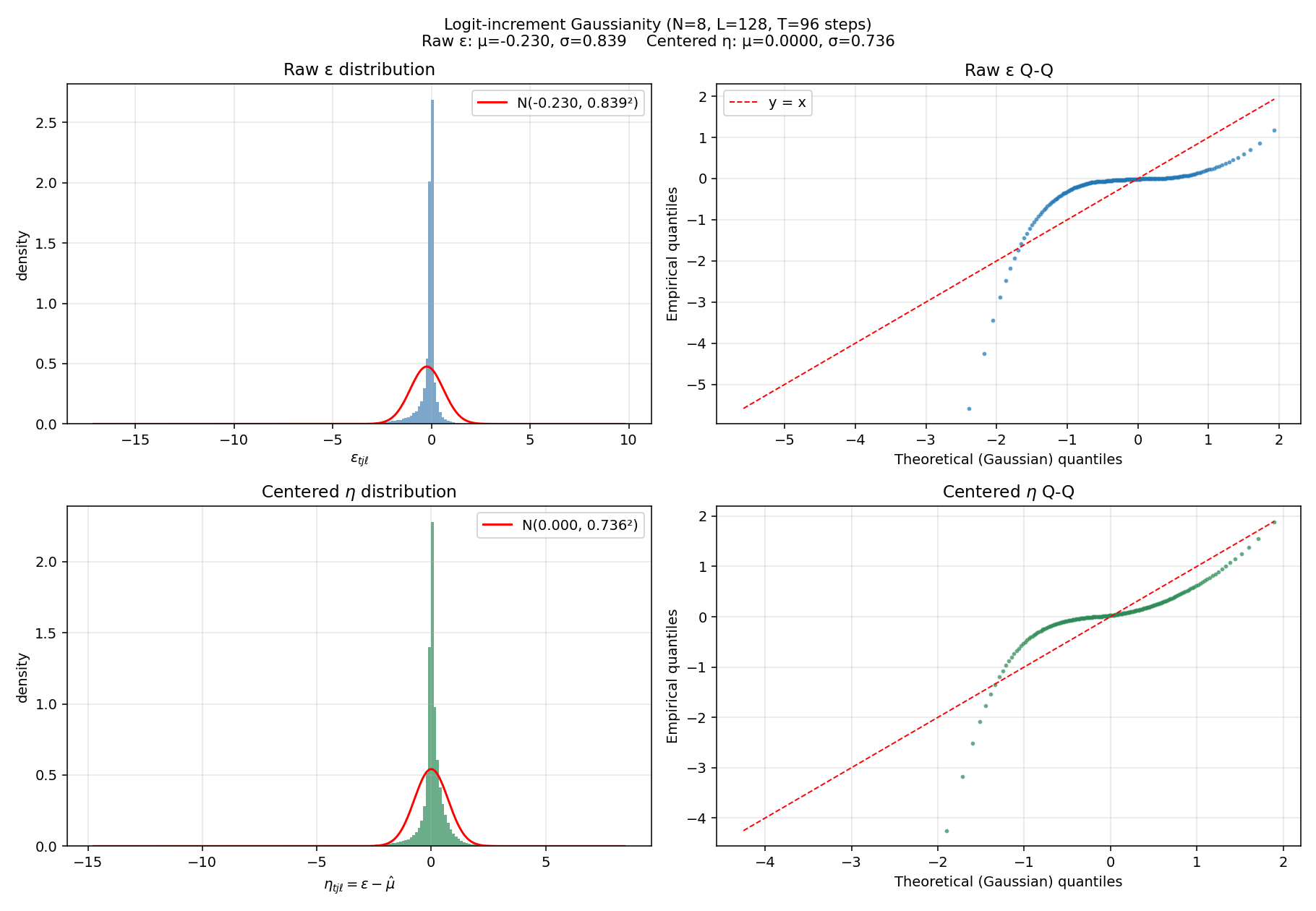}
\caption{Logit-increment distribution on TinyStories MDLM. Top row: raw $\varepsilon_{tj\ell}$. Bottom row: residual $\eta_{tj\ell} = \varepsilon - \hat\mu_{tj\ell}$ after subtracting the per-($t,j,\ell$) sample mean. Centring concentrates the bulk near a Gaussian; the remaining heavy left tail corresponds to unmasking events. Overall: the factorisation required by the theory holds directionally; cross-position covariance is empirically near zero ($\bar\rho \approx 10^{-5}$), and the central mass of $\eta$ is near-Gaussian with time-varying (but bounded) $\sigma_t$.}
\label{fig:temporal-consistency}
\end{figure}

Across the 128 denoising steps, $\sigma_t$ fluctuates between $\sim$0.2 and $\sim$1.6 with mean 0.71 (not constant, but bounded). The cross-position covariance is the strongest empirical result: for every token tested, the empirical $\bar\rho$ of the residual $\eta$ is on the order of $10^{-5}$, i.e., essentially zero relative to its diagonal variance. This is the component of the assumption that drives the per-position factorisation in Section~\ref{sec:method}, and it holds very cleanly on our model.

\paragraph{Logit trajectories.} To visualise how individual logits evolve under the reverse process, we plot $\log p_\theta(x_t^\ell = j \mid x_{t+1})$ directly from the backbone (before the SUBS parameterisation rewrites unmasked positions to identity, which would otherwise produce spurious step-changes at unmasking events). Figure~\ref{fig:logit-trajectories} shows the top-3 tokens per position across all $L=32$ positions of a single sample, colour-coded by position. The trajectories are continuous and smoothly varying: most tokens stabilise into distinct log-probability bands by step ${\sim}20$--$40$ and then drift slowly for the remainder of the denoising process. This is precisely the picture Assumption~\ref{as:temporal_consistency} encodes: a slow per-($t,j,\ell$) drift $\mu_{tj\ell}$ plus small noise $\eta_{tj\ell}$.

\begin{figure}[h]
\centering
\includegraphics[width=0.95\linewidth]{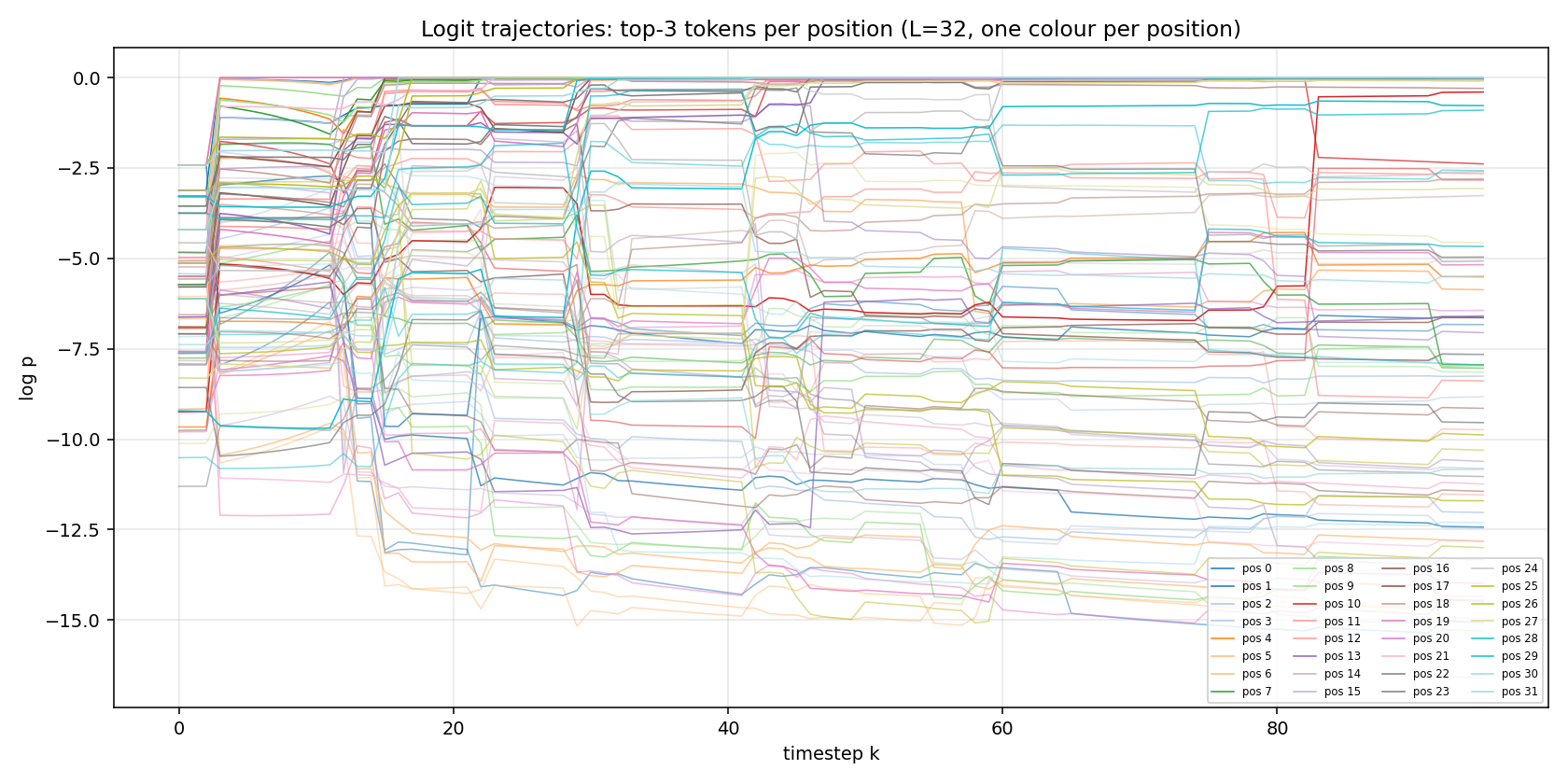}
\caption{Raw backbone log-probabilities $\log p_\theta(x_t^\ell = j \mid x_{t+1})$ for the top-3 tokens at each of $L=32$ positions, across $T=96$ denoising steps (one colour per position; transparency encodes rank within the position). Logits evolve smoothly, with no sharp transitions; we plot pre-parameterisation values, bypassing the SUBS rewrite that forces unmasked positions to identity. Different positions cluster at different log-probability bands but do not co-move strongly, consistent with the low-cross-position-covariance finding.}
\label{fig:logit-trajectories}
\end{figure}

Figure~\ref{fig:logit-heatmap} complements this with a vocabulary-wide view: for eight representative positions, we show a heatmap of $\log p_\theta(x_t^\ell = j \mid x_{t+1})$ with the vocabulary sorted by final-step probability. The top tokens (top rows) brighten smoothly as $t$ decreases; lower-rank tokens fade continuously. No sharp bands appear: the model's beliefs concentrate gradually rather than via discrete commitments at specific steps.

\begin{figure}[h]
\centering
\includegraphics[width=0.95\linewidth]{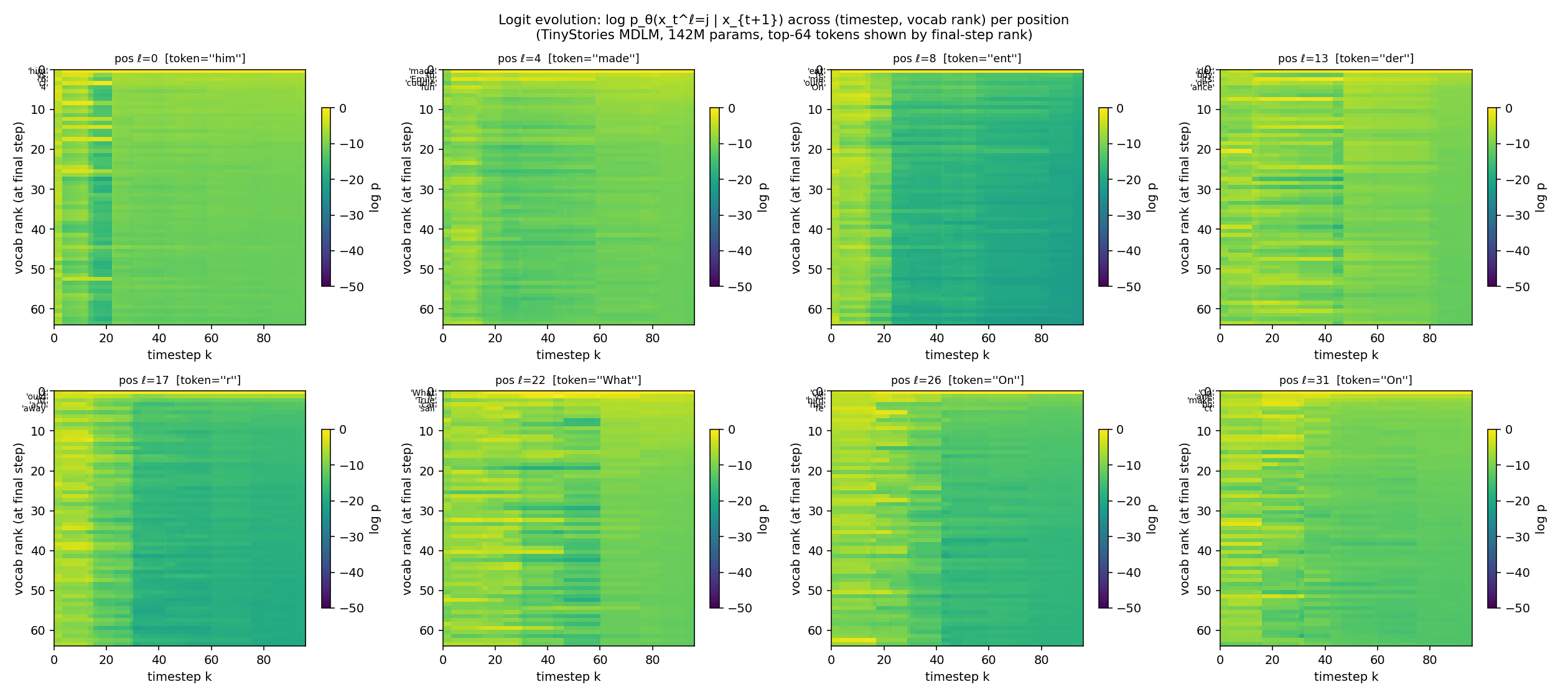}
\caption{Vocabulary-wide logit evolution for eight representative positions. $x$: denoising step $k$. $y$: vocabulary rank sorted by final-step log-probability (top rows = eventually-high-probability tokens). Colour: $\log p_\theta$. Top-$64$ tokens shown per position; rest are uniformly low. Concentration is gradual, supporting the smooth-drift picture of Assumption~\ref{as:temporal_consistency}.}
\label{fig:logit-heatmap}
\end{figure}

Taken together, Figures~\ref{fig:logit-trajectories}--\ref{fig:logit-heatmap} illustrate that (i) logits evolve continuously without large jumps within the reverse process, (ii) different positions follow similar \emph{shape} dynamics but with uncorrelated residuals (consistent with $\bar\rho \approx 0$), and (iii) the per-position factorisation used in our derivation is not merely a convenient idealisation but reflects actual behaviour of a trained MDLM on text.

\paragraph{Cross-domain comparison.} To test whether the factorisation extends beyond text, we repeated the analysis on the SMILES MDLM ($21.9$M params, vocab $59$, $L=72$) used for the molecular experiments. Figure~\ref{fig:temporal-consistency-mol} shows the cross-position covariance of the centred residual $\eta$ for four tokens. In contrast to the near-zero values observed on text ($\bar\rho \approx 10^{-5}$), molecular covariance is substantial and spatially structured: $\bar\rho$ ranges from $0.03$ for the top-1 token up to $0.39$ for sub-top tokens, with a clear banded diagonal indicating that nearby positions have correlated $\eta$. This reflects the strong local constraints of SMILES syntax (valence, ring closures, bracket pairing) which induce inter-position dependence that text does not exhibit. Kurtosis of the centred $\eta$ also rises from $29$ (text) to $80$ (molecular), indicating substantially heavier tails.

\begin{figure}[h]
\centering
\includegraphics[width=0.95\linewidth]{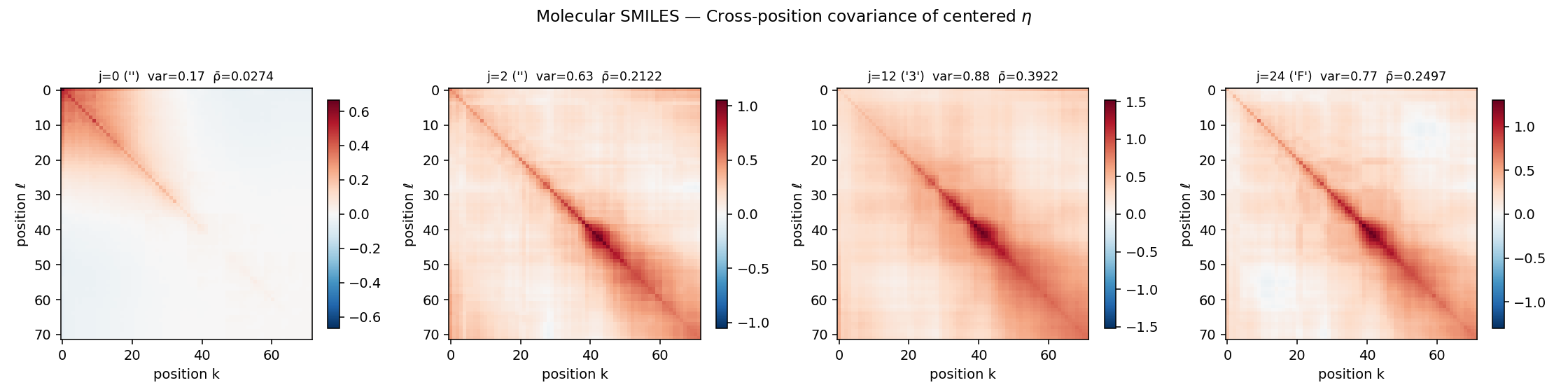}
\caption{Molecular SMILES MDLM: cross-position covariance of centred $\eta$ for four tokens. Unlike text, molecular shows strong local correlation structure ($\bar\rho$ up to $0.39$) consistent with chemistry's valence and ring-closure constraints. The assumption's literal bound on $\rho$ is violated in this domain, though constrained generation still succeeds empirically (Table~\ref{tab:guidance}), suggesting the Statement~\ref{th:regret} bound is loose in practice for structured domains.}
\label{fig:temporal-consistency-mol}
\end{figure}

This does not invalidate our empirical results on molecular (Table~\ref{tab:guidance}), since SPDD still improves constraint satisfaction over unconstrained generation at competitive KL, but it does indicate that the parametric form of Assumption~\ref{as:temporal_consistency} is an approximation, and the theoretical bound in Statement~\ref{th:regret} is likely loose for strongly-structured discrete domains. A principled sharpening of the assumption (e.g., allowing a block-diagonal covariance reflecting local structure) is left for future work.

\paragraph{Does the bound actually hold in practice?} The natural follow-up is whether the violations of the assumption's parametric form translate into violations of the downstream guarantee. Figure~\ref{fig:bound-vs-empirical} compares the empirical per-step objective $\pi_t$ attained by SPDD on a real ocean-steering run (LHS of Statement~\ref{th:regret}) against the theoretical lower bound (RHS), as a function of the denoising step. The bound holds empirically at \textbf{96.9\% of steps}; the few small negative excursions near $t=0$ are order $-3$ in magnitude compared to peak positive slack of ${+}60$. The bound is \emph{loose but valid}: total empirical reward is $-936$ versus a lower bound of $-1281$ (27\% slack). This matches the expected behaviour: assumption violations loosen the worst-case bound, but the empirical trajectory of a trained SPDD sampler still sits comfortably above the guarantee.

\begin{figure}[h]
\centering
\includegraphics[width=0.95\linewidth]{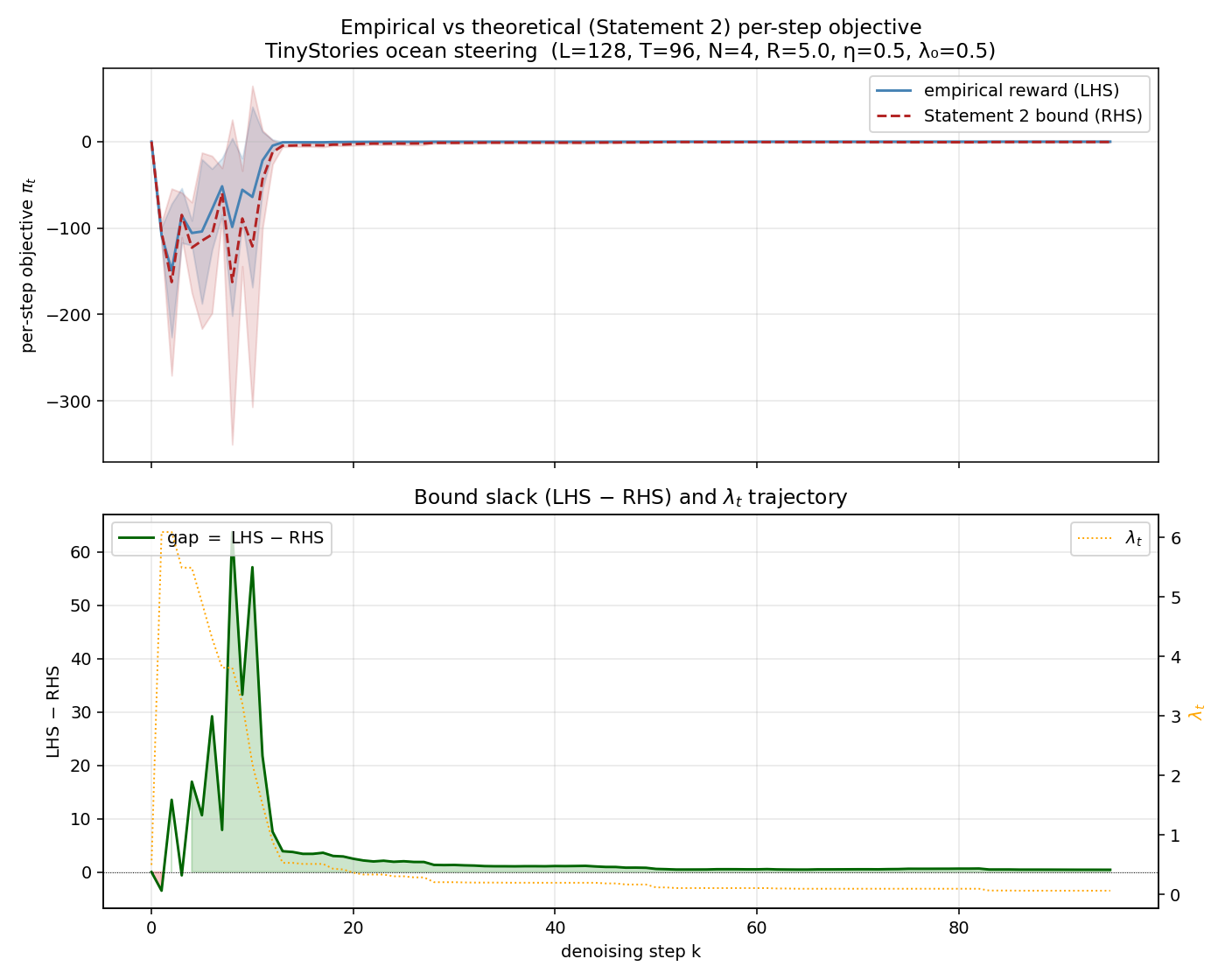}
\caption{Per-step objective $\pi_t$ under SPDD ocean steering ($L{=}128$, $T{=}96$, $N{=}4$, $R{=}5$, $\eta{=}0.5$, $\lambda_0{=}0.5$). Top: empirical reward (blue) and Statement~\ref{th:regret} lower bound (red, dashed) over denoising steps. Bottom: slack (green = bound holds; red = bound violated) and $\lambda_t$ trajectory (orange). The bound holds at 96.9\% of steps; total empirical reward exceeds the bound by 27\%. The dynamics concentrate in the first ${\sim}15$ steps; after constraint satisfaction, $\lambda_t \approx 0$ and both sides collapse to zero.}
\label{fig:bound-vs-empirical}
\end{figure}

Taken together, the analyses in this subsection make three claims at different levels of strength. (i) The cross-position factorisation holds cleanly on text-like domains ($\bar\rho \approx 10^{-5}$) and is substantially weaker on structured domains like molecular ($\bar\rho$ up to $0.39$). (ii) The Gaussian / bounded-$\sigma^2$ parametric form is an approximation: centred residuals have kurtosis $29$ on text and $80$ on molecular, and $\sigma_t$ varies $8\times$ across denoising steps. (iii) Despite (i)--(ii), the Statement~\ref{th:regret} bound is empirically valid on actual SPDD runs at $>$96\% of denoising steps, with substantial slack. The theoretical guarantee is thus robust to the assumption-violations we observe.

\subsection{Zero-Shot Transfer to LLaDA-8B}\label{app:llada}

To test whether SPDD transfers to a substantially larger backbone without modification, we apply the same primal-dual update at decode time to LLaDA-8B~\citep{nie2024llada}, an absorbing-state diffusion model trained at the 8B scale. No fine-tuning or auxiliary training is performed; only the constraint-dependent logit bias is added at each denoising step. Table~\ref{tab:transfer} reports the result on the same TinyStories ocean target ($R{=}10$, $N{=}1{,}000$, sampling temperature $1.0$).

\begin{table}[h]
\centering
\caption{Zero-shot transfer of SPDD to LLaDA-8B on the TinyStories ocean target ($R{=}10$, $N{=}1{,}000$). KL is unigram divergence from LLaDA-8B's own unconstrained sampling at temperature $1.0$.}
\label{tab:transfer}
\begin{tabular}{lccccc}
\toprule
Method & Target (avg) & Pass\% & PPL$\downarrow$ & Dist-2$\uparrow$ & KL$\downarrow$ \\
\midrule
SPDD on LLaDA-8B & 14.1 & 100\phantom{$\pm$00.0}\% & \phantom{0}8.6 & 0.125 & 1.710 \\
\bottomrule
\end{tabular}
\end{table}

The pass rate transfers cleanly: the same per-token scoring interface that drives the trained MDLM also steers a frozen 8B model to satisfy the lexical constraint without any model surgery. The unusually low PPL is misleading on its own. Unconstrained LLaDA-8B at temperature $1.0$ already exhibits repetition under our minimal sampling protocol, and SPDD's bias amplifies this into degenerate sequences (e.g., \emph{ocean ocean ocean \ldots}) that score well under GPT-2 perplexity but score poorly under Dist-2 (0.125, well below the unconstrained MDLM at 0.187). We therefore treat this row as evidence that the per-token scoring interface is backbone-agnostic, not as a competitive fluency claim against the trained MDLM in Table~\ref{tab:ocean}. Closing the diversity gap on large absorbing-state backbones likely requires standard repetition controls (temperature scheduling, frequency penalties, or top-$k$ filtering) layered on top of SPDD; we leave that as future work.

\subsection{Molecular Generation Details}\label{app:mol-details}

Table~\ref{tab:guidance} reports the full MW $\geq 350$ sweep that underlies the molecular column of Table~\ref{tab:unified}. Each method is reported at one representative operating point (the highest Pass\% within a $\geq 50\%$ validity envelope, except for Static where we show three $\alpha$ values to expose the $\alpha$ frontier directly). KL is the unigram KL between each method's samples and the unconstrained samples from the same model. Pass\% (column ``Above $350$'') is the fraction of generated samples whose RDKit-parsed heavy-atom MW is at or above the threshold.

\begin{table}[htbp]
\centering
\caption{Molecular MW $\geq 350$ guidance, $N{=}1{,}000$, 21.9M SMILES backbone. SPDD: accumulated slack, $\eta{=}1.0$, $\lambda_0{=}0.1$, per-position bias (best of a 7-config sweep).}
\label{tab:guidance}
\begin{tabular}{lccccc}
\toprule
Method & Valid & QED & MW (mean) & Above $350$ & KL$\downarrow$ \\
\midrule
Unconstrained & 59.9\% & 0.648 & 330 & 29.40\% & --- \\
Static ($\alpha{=}1$) & 60.9\% & 0.647 & 338 & 31.10\% & 0.005 \\
Static ($\alpha{=}2$) & 59.1\% & 0.649 & 348 & 33.10\% & 0.008 \\
Static ($\alpha{=}5$) & 59.6\% & 0.598 & 382 & 41.40\% & 0.019 \\
D-CBG ($\gamma{=}2$) & 57.4\% & 0.606 & 356 & 35.30\% & 0.014 \\
D-CBG ($\gamma{=}5$) & 54.8\% & 0.544 & 412 & 42.80\% & 0.048 \\
CDD ($\tau{=}0.3$, $o{=}5,i{=}10$) & 57.9\% & 0.617 & 372 & 38.90\% & 0.019 \\
CDD ($\tau{=}0.5$, $o{=}5,i{=}10$) & 59.7\% & 0.624 & 382 & 42.20\% & 0.025 \\
\midrule
SPDD (Ours) & 49.9\% & 0.504 & 473 & 48.50\% & 0.096 \\
\bottomrule
\end{tabular}
\end{table}

\paragraph{Multi-constraint composition.}

A key structural advantage of the primal-dual formulation is native composition of multiple soft above-target objectives: each constraint contributes its own Lagrange multiplier $\lambda_i$ (with its own step size $\eta_i$), and Equation~\ref{eq:sampling} sums their additive logit biases. Tables~\ref{tab:guidance-dual} and~\ref{tab:guidance-dual-tight} evaluate two simultaneous objectives on the same 21.9M backbone, with SPDD taking both objectives jointly via $\{\text{MW}, \#(N+O)\}$ per-token score vectors and per-objective $(\eta_i, \lambda_{0,i})$ rescaled by each score vector's range. We compare against three classes of multi-constraint baselines: \emph{Static bias} adds a scalar over the sum of the two per-token score vectors; \emph{D-CBG (joint)} trains a single classifier on the joint binary label $\big[\text{MW}\geq T_{\text{MW}} \wedge \#(N+O)\geq T_{N{+}O}\big]$; \emph{D-CBG (MW only)} and \emph{CDD} optimise the MW objective alone (the latter has no native multi-constraint extension).

\begin{table}[htbp]
\centering
\caption{Dual above-target guidance: MW~$\geq 350$ AND \#(N+O)~$\geq 5$, $N{=}1{,}000$, 21.9M SMILES backbone. Joint Pass\% requires both constraints. SPDD: best of an 11-config sweep over $\eta{\in}\{0.1, 0.2, 0.5, 1.0, 2.0, 5.0\}$ and $\lambda_0{\in}\{0.05, 0.1, 0.5\}$, per-objective rescaled by score range; the picked config is $\eta{=}0.2$, $\lambda_0{=}0.05$, accumulated slack, per-position bias.}
\label{tab:guidance-dual}
\begin{tabular}{lcccccc}
\toprule
Method & Valid & Joint Pass\% & MW Pass\% & \#(N+O) Pass\% & mean \#(N+O) & KL$\downarrow$ \\
\midrule
Unconstrained                       & 61.3\% & 20.2\% & 21.4\% & 40.3\% & 5.42 & 0.005 \\
Static ($\alpha{=}2$)               & 54.0\% & 27.8\% & 27.8\% & 48.7\% & 8.98 & 0.038 \\
D-CBG ($\gamma{=}5$, MW only)       & 55.2\% & 36.0\% & 38.3\% & 40.9\% & 5.91 & 0.038 \\
D-CBG ($\gamma{=}5$, joint)         & 53.4\% & 36.5\% & 39.0\% & 39.2\% & 5.96 & 0.041 \\
CDD ($\tau{=}0.5$, MW only)         & 56.1\% & 34.2\% & 36.4\% & 46.7\% & 6.49 & 0.033 \\
\midrule
SPDD (Ours)                         & 38.7\% & 38.1\% & 38.1\% & 38.7\% & 7.32 & 0.073 \\
\bottomrule
\end{tabular}
\end{table}

\begin{table}[htbp]
\centering
\caption{Dual above-target guidance with the tighter \#(N+O)~$\geq 8$ second threshold (same MW threshold), $N{=}1{,}000$, 21.9M SMILES backbone. The tight second threshold breaks the natural MW$\leftrightarrow$\#(N+O) correlation: MW-only baselines plateau well below SPDD, and only the \emph{joint-label} D-CBG variant or Static $\alpha{=}2$ remain competitive. SPDD: best of the same 11-config sweep ($\eta{=}0.1$, $\lambda_0{=}0.05$, accumulated, per-position).}
\label{tab:guidance-dual-tight}
\begin{tabular}{lcccccc}
\toprule
Method & Valid & Joint Pass\% & MW Pass\% & \#(N+O) Pass\% & mean \#(N+O) & KL$\downarrow$ \\
\midrule
Unconstrained                       & 61.0\% & 10.4\% & 21.8\% & 13.0\% & 5.41 & 0.004 \\
Static ($\alpha{=}2$)               & 54.9\% & 27.0\% & 28.7\% & 39.9\% & 8.83 & 0.035 \\
D-CBG ($\gamma{=}5$, MW only)       & 52.8\% & 17.1\% & 40.3\% & 17.2\% & 6.16 & 0.044 \\
D-CBG ($\gamma{=}5$, joint)         & 53.7\% & 25.1\% & 34.3\% & 27.9\% & 7.12 & 0.029 \\
CDD ($\tau{=}0.5$, MW only)         & 56.7\% & 16.8\% & 35.7\% & 18.8\% & 6.44 & 0.030 \\
\midrule
SPDD (Ours)                         & 35.9\% & 34.4\% & 35.4\% & 34.9\% & 8.83 & 0.070 \\
\bottomrule
\end{tabular}
\end{table}

At the looser \#(N+O)~$\geq 5$ threshold (Table~\ref{tab:guidance-dual}) the natural correlation between MW and \#(N+O) in ZINC20 lets MW-only baselines satisfy both constraints passively; SPDD reaches the highest joint Pass\% (38.1\%) by a small $\sim 1.5$-point margin over the strongest baseline (D-CBG joint at 36.5\%), at the cost of lower validity (38.7\% vs.\ 53.4\%) since SPDD pushes the distribution further into both constraint regions.

The tighter \#(N+O)~$\geq 8$ regime (Table~\ref{tab:guidance-dual-tight}) breaks the natural correlation. The MW-only baselines that ignored the heteroatom signal collapse to $\sim 17\%$ joint Pass\%, while SPDD continues to lift both axes (joint $34.4\%$, mean \#(N+O)~$\approx 8.8$). The \emph{joint-label} D-CBG variant, given the conjunction as its training signal, recovers some of the gap (joint $25.1\%$ at higher validity than SPDD), making it the strongest of the multi-constraint baselines, but SPDD still wins joint Pass\% by $\sim 9$ points. The validity penalty for SPDD is real: pushing both axes simultaneously strains chemical syntax. Lowering $\eta$ trades joint Pass\% for KL and validity along the same controllable frontier as the single-constraint case; even at the smallest $\eta$ in our sweep ($\eta{=}0.1$), SPDD's KL ($\sim 0.07$) sits roughly $2\times$ above the baselines' KL ($0.03$--$0.04$), reflecting the cost of steering both axes through the model's distribution rather than only one.

The structural advantage of SPDD over D-CBG-joint is composability rather than absolute Pass\%: SPDD adds one Lagrangian multiplier per constraint and uses the existing per-token scoring functions; D-CBG-joint instead requires re-training a classifier each time the constraint set changes, which does not scale gracefully with the number of constraints. CDD has no native joint extension and remains a single-constraint baseline. Per-objective $\eta_i$ and $\lambda_{0,i}$ are essential for SPDD here because the raw score vectors (atomic mass $\sim 12$--$127$ vs.\ binary $\{0,1\}$ for heteroatoms) live on very different scales; without rescaling, the MW objective dominates the dynamics. Extending the framework to less-correlated quantities (e.g., aromatic-ring counts) is a near-term direction.

\paragraph{Structural Constraints.}

Table~\ref{tab:structural} evaluates soft multi-token subsequence guidance using \textbf{SubsequenceScorer}, which biases toward generating a target subsequence (benzene ring \texttt{c1ccccc1}) at an unknown position via context-aware per-position bias scores, unlike hard inpainting which requires known positions. SPDD uses $\eta{=}2.0$, $\lambda_0{=}0.5$, instantaneous slack.

\begin{table}[htbp]
\centering
\caption{Soft substructure guidance via SubsequenceScorer ($N{=}1{,}000$, SMILES, 21.9M backbone). Constraint satisfaction measured as fraction of generated SMILES containing the target subsequence (benzene ring \texttt{c1ccccc1}).}
\label{tab:structural}
\begin{tabular}{lccc}
\toprule
Method & Validity & Benzene rate & Cost \\
\midrule
Unconstrained & 64.4\% & 29.9\% & $1\times$ \\
SPDD SubsequenceScorer & 56.9\% & 40.3\% & $1\times$ \\
\bottomrule
\end{tabular}
\end{table}

SPDD lifts benzene containment by 10 points (29.9\%$\to$40.3\%, $1.35\times$ lift) with modest validity degradation (64.4\%$\to$56.9\%), extending the framework from single-token constraints to multi-token structural motifs at unknown positions.

\paragraph{MW validity trade-off and qualitative examples.}

Sweeping $\eta$ and the slack mode traces a Pareto frontier in Pass\% versus chemical validity: weaker guidance keeps validity near the unconstrained level but undershoots the threshold, while more aggressive guidance lifts Pass\% but degrades validity through malformed SMILES. Table~\ref{tab:guidance} summarises the operating points along this frontier; the corresponding sweep configurations are saved alongside the released results. Figure~\ref{fig:mol-examples} shows the qualitative effect on individual molecules: row (i) spans the natural MW distribution of ZINC20, row (ii) shifts mass into the upper tail while preserving drug-like scaffolds, and row (iii) additionally enforces nitrogen/oxygen richness without visibly collapsing structural diversity.

\begin{figure}[htbp]
\centering
\includegraphics[width=\linewidth]{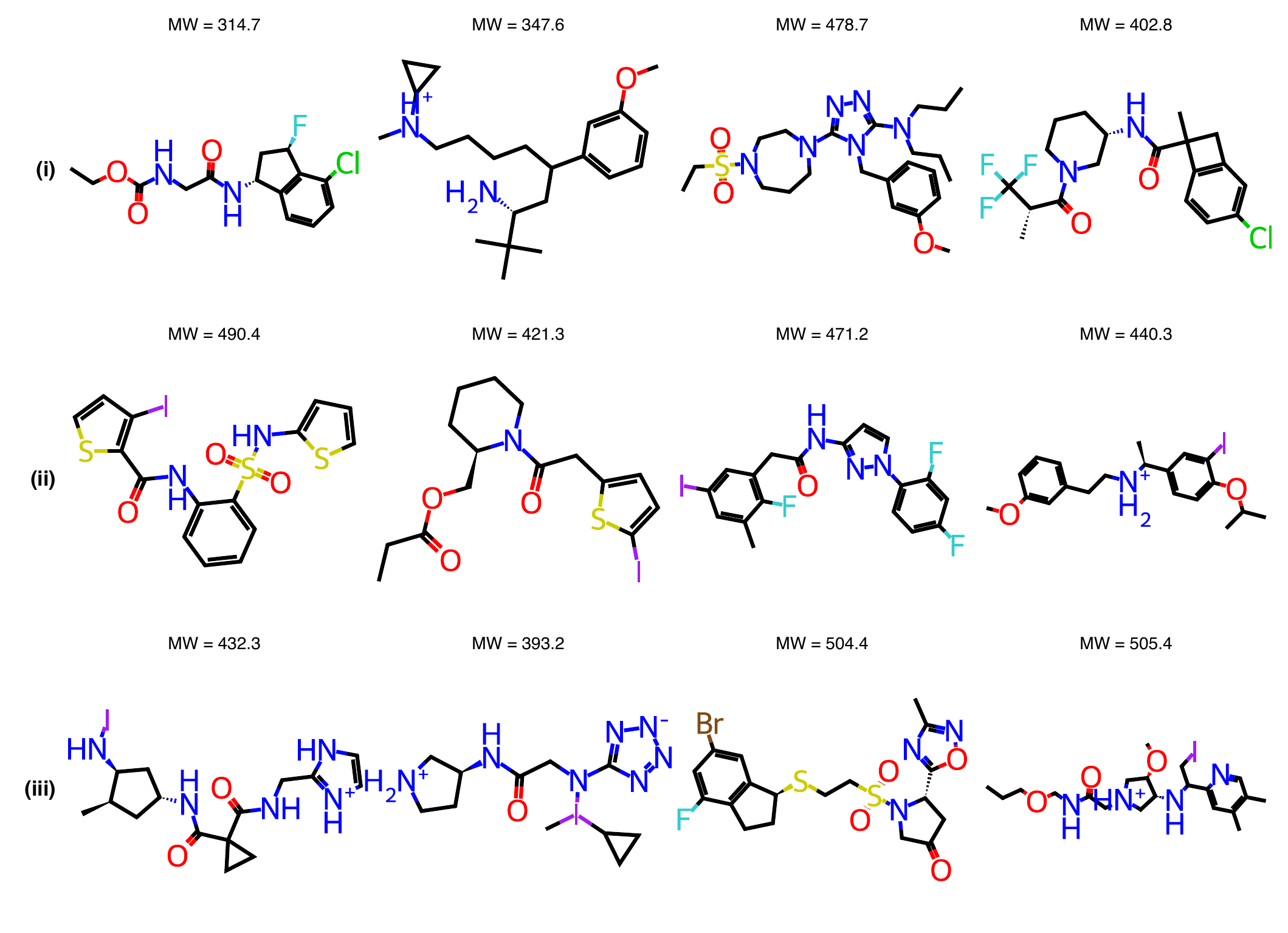}
\caption{Example molecules from the 21.9M SMILES backbone. \textbf{(i)} Unconstrained generation. \textbf{(ii)} SPDD with MW~$\geq 350$\,Da. \textbf{(iii)} SPDD with the dual constraint MW~$\geq 350$\,Da and \#(N+O)~$\geq 5$. Molecular weights are annotated above each structure.}
\label{fig:mol-examples}
\end{figure}

\subsection{Playlist Generation Details}\label{app:playlist-exp-details}

\begin{table}[htbp]
\centering
\small
\caption{Playlist subgenre steering, 5-subgenre mean, $N{=}1{,}000$ sessions per subgenre. NDCG/Recall reported for completeness; see §\ref{sec:playlist} for why they are not the primary metric.}
\label{tab:playlist-results}
\begin{tabular}{lrrrrrr}
\toprule
Method & NDCG@30 & Recall@30 & Sat.\%$\uparrow$ & Div.$\uparrow$ & IDist.$\uparrow$ & KL$\downarrow$ \\
\midrule
Unconstrained & 0.378 & 0.372 & \phantom{0}1.1\% & 0.497 & 0.987 & 0.000 \\
Static ($\alpha{=}2.0$) & 0.504 & 0.372 & \phantom{0}1.7\% & 0.494 & 0.987 & 0.006 \\
D-CBG ($\gamma{=}5$) & 0.488 & 0.354 & \phantom{0}1.9\% & 0.481 & 0.979 & 0.118 \\
CDD ($\tau{=}0.5$) & 0.439 & 0.299 & \phantom{0}9.2\% & 0.368 & 0.986 & 0.782 \\
\midrule
SPDD (Ours, $\eta{=}0.2$, frontload) & 0.306 & 0.357 & \phantom{0}9.3\% & 0.513 & 0.969 & 0.355 \\
\bottomrule
\end{tabular}
\end{table}

\paragraph{Setup Details.}

Each token inherits genre metadata from its member tracks. For a target genre $g$, we compute per-token scores as the fraction of member tracks tagged with $g$ (exact match):
\begin{equation}
    b_g(j) = \frac{|\{t \in \text{members}(j) : g \in \text{genres}(t)\}|}{|\text{members}(j)|}
\end{equation}
We target specific subgenres rather than broad categories. Because tracks carry multiple genre tags and LSH clusters mix tags, a single statistic does not capture how the constraint signal is distributed across the vocabulary. We therefore report two metrics per genre (Table~\ref{tab:playlist-coverage}): the \emph{breadth} of the signal (the fraction of tokens with $b_g(j){>}0$), and its \emph{concentration} (the mean of $b_g(j)$ across the vocabulary, plus the fraction of tokens with $b_g(j){\geq}0.5$, i.e., clusters whose majority is tagged $g$). The two metrics decouple sharply: 19\% of tokens carry \emph{some} k-pop signal, but only 0.5\% of tokens are majority-k-pop, and the average per-token score is 0.019 (i.e., a typical cluster has ${\sim}1$--$2$ k-pop tracks out of ${\sim}50$). Sparse genres are sparse on \emph{both} axes: thrash metal has 2.4\% breadth and 0.0022 mean. The breadth percentages overlap across genres and need not sum to $\leq 100\%$, but the concentration metrics do not have this property and remain meaningful as absolute discretisation statistics. Breadth captures how many tokens the constraint touches at all; concentration captures how strongly each cluster is dominated by the target subgenre. The two metrics decouple sharply because LSH clusters mix tags.

\begin{table}[htbp]
\centering
\small
\caption{Per-genre coverage statistics on the 3{,}887-token playlist vocabulary. \emph{Breadth} ($b_g{>}0$) measures how widely the constraint signal is spread; \emph{Concentration} (mean $b_g$, fraction of tokens with $b_g{\geq}0.5$, max $b_g$) measures how strongly any single token leans toward the genre. The signal is broadly diffuse ($\geq80\%$ of tokens have $b_g{=}0$ on every genre, and only $0.1$--$0.5\%$ of tokens are majority-genre), reflecting the multi-tag nature of editorial track metadata under LSH clustering of audio embeddings.}
\label{tab:playlist-coverage}
\begin{tabular}{lrrrrr}
\toprule
Genre & $|b_g{>}0|$ & $\mathbb{E}\,b_g$ & $|b_g{\geq}0.1|$ & $|b_g{\geq}0.5|$ & $\max b_g$ \\
\midrule
reggaeton    & 18.9\% & 0.0169 & 5.7\% & 0.3\% & 1.00 \\
k-pop        & 18.7\% & 0.0185 & 5.8\% & 0.5\% & 1.00 \\
bossa nova   &  9.8\% & 0.0029 & 0.6\% & 0.1\% & 0.50 \\
synthwave    &  5.3\% & 0.0023 & 0.6\% & 0.1\% & 0.85 \\
thrash metal &  2.4\% & 0.0022 & 0.6\% & 0.1\% & 1.00 \\
\bottomrule
\end{tabular}
\end{table}

We report \emph{bucket-level} and \emph{track-level} genre satisfaction (the latter after nearest-centroid track selection), embedding coherence, token diversity, and inter-playlist Jaccard distance.

\paragraph{Additional Subgenre Results.}

Table~\ref{tab:playlist-results-full} reports appendix exact-match runs with an expanded metric set for k-pop (19\% token coverage) and synthwave (8\% coverage).

\begin{table}[htbp]
\centering
\caption{Additional playlist subgenre results ($N{=}200$, exact match). Bucket and Track values are percentages; Coher.\ is embedding coherence, Div.\ is token diversity, and IDist.\ is inter-playlist Jaccard diversity.}
\label{tab:playlist-results-full}
\begin{tabular}{llrrrrr}
\toprule
Genre & Method & Bucket (\%)$\uparrow$ & Track (\%)$\uparrow$ & Coher. & Div. & IDist. \\
\midrule
\multirow{5}{*}{k-pop} & Unconstrained & 51.0 & 5.2\tiny$\pm$1.1 & 0.678 & 0.475 & 0.977 \\
& Static ($\alpha{=}5$) & 67.1 & 20.7\tiny$\pm$3.1 & 0.691 & 0.452 & 0.967 \\
& D-CBG ($\gamma{=}5$) & 94.0 & 6.1 & 0.814 & 0.317 & 0.861 \\
& CDD ($\tau{=}0.5$) & 97.8 & 50.3 & 0.930 & 0.090 & 0.910 \\
& SPDD ($\eta{=}5$) & 79.6 & 36.7\tiny$\pm$3.1 & 0.716 & 0.380 & 0.925 \\
\midrule
\multirow{5}{*}{synthwave} & Unconstrained & 36.0 & 1.0\tiny$\pm$0.4 & 0.678 & 0.475 & 0.977 \\
& Static ($\alpha{=}5$) & 48.5 & 3.4\tiny$\pm$0.8 & 0.672 & 0.387 & 0.971 \\
& D-CBG ($\gamma{=}5$) & 93.4 & 7.8 & 0.806 & 0.218 & 0.879 \\
& CDD ($\tau{=}0.5$) & 100.0 & 24.8 & 0.998 & 0.021 & 0.028 \\
& SPDD ($\eta{=}5$) & 73.6 & 20.9\tiny$\pm$2.9 & 0.619 & 0.329 & 0.952 \\
\bottomrule
\multicolumn{7}{l}{\footnotesize Requires training a separate classifier per genre ($2\times$ cost for D-CBG, $5$--$20\times$ for CDD).}
\end{tabular}
\end{table}

These appendix rows expose the same satisfaction--diversity tension with a wider metric set. On synthwave, CDD achieves 100\% bucket satisfaction but diversity drops to 0.02 with inter-playlist Jaccard distance of 0.03: all playlists are nearly identical. SPDD reaches lower track satisfaction but maintains substantially higher diversity and inter-playlist distance.

\paragraph{Slack modes and $\lambda$ dynamics.}\label{app:playlist-slack-modes}

The slack-mode pattern observed on text generalises cleanly to playlists (Table~\ref{tab:playlist-slack}). Instantaneous slack maximises track satisfaction (81.7\% at $\eta{=}5.0$) but compresses diversity (0.33$\to$0.21 as $\eta$ rises); accumulated slack sacrifices peak satisfaction for monotonic controllability (7.6\%$\to$46.8\%) while preserving diversity (${\sim}0.46$); optimistic (OMD) slack tracks instantaneous at high $\eta$ (63.4\% at $\eta{=}0.2$, 81.0\% at $\eta{=}5.0$). On the harder thrash-metal setting (3.3\% coverage) the same ordering holds (instantaneous 24.2\%, early 19.0\%, optimistic 14.7\%, accumulated 17.1\% at $\eta{=}5$), confirming the pattern is genre-independent. Per-step dynamics (Figure~\ref{fig:playlist-trace}) explain the trade-off: instantaneous slack pushes $\lambda$ to the clamp ceiling immediately, committing all target tokens in the first ${\sim}20$ denoising steps and leaving no room for diversity, whereas optimistic slack rises gradually as the model's forecast $\hat{E}_k$ absorbs context.

\paragraph{Front-loaded positional bias.}\label{app:playlist-frontload}
Table~\ref{tab:playlist-frontload-full} reports a full $\eta$ sweep for front-loaded positional bias across all five paper subgenres ($N{=}1{,}000$ held-out sessions per genre, mean across genres). The pattern is consistent across $\eta$: front-loading concentrates the genre signal at the top of the playlist (TrGNDCG $0.77$--$0.82$ vs.\ uniform's $0.54$--$0.67$, a ${\sim}25$pp lift) and slightly improves Recall@30 at $\eta{=}0.5$. The cost is concentrated at NDCG@5: the front-loaded variant places the catalogue's most-genre tokens at positions 1--5, which often differ from the held-out user's specific tracks (NDCG@5 drops from ${\sim}0.10$--$0.19$ to $0.03$--$0.05$). NDCG@30 is somewhat lower at $\eta{=}0.5$ (0.252 vs.\ uniform's 0.295) but the gap narrows as $\eta$ grows. The variant is therefore the right knob when genre purity at the playlist head matters more than personalised matching to a specific user's listening history (themed/radio playlists, mood-based generation).

\begin{table}[htbp]
\centering
\small
\caption{Front-loaded vs.\ uniform SPDD bias, $\eta$ sweep, mean across reggaeton, k-pop, synthwave, bossa nova, and thrash metal ($N{=}1{,}000$ held-out sessions per genre). \emph{TrGNDCG} (track-level genre NDCG) measures whether target-genre tracks land at high-ranked positions. NDCG@5/@10/@30 measure recommendation relevance against the held-out user continuation at three cutoffs; Sat.\% is bucket-level target-genre satisfaction.}
\label{tab:playlist-frontload-full}
\begin{tabular}{llrrrrrr}
\toprule
Variant & $\eta$ & NDCG@5$\uparrow$ & NDCG@10$\uparrow$ & NDCG@30$\uparrow$ & Recall@30$\uparrow$ & Sat.\%$\uparrow$ & TrGNDCG$\uparrow$ \\
\midrule
SPDD uniform   & 0.5 & 0.189 & 0.235 & \textbf{0.295} & 0.283 & 67.0\% & 0.543 \\
SPDD uniform   & 1.0 & 0.134 & 0.165 & 0.194 & 0.230 & 76.4\% & 0.605 \\
SPDD uniform   & 2.0 & 0.117 & 0.138 & 0.162 & 0.212 & 78.0\% & 0.632 \\
SPDD uniform   & 5.0 & 0.098 & 0.117 & 0.141 & 0.203 & 77.9\% & 0.669 \\
\midrule
SPDD frontload & 0.5 & 0.052 & 0.093 & 0.252 & \textbf{0.296} & 65.0\% & \textbf{0.823} \\
SPDD frontload & 1.0 & 0.025 & 0.054 & 0.152 & 0.230 & 75.4\% & 0.815 \\
SPDD frontload & 2.0 & 0.040 & 0.065 & 0.128 & 0.212 & 77.4\% & 0.795 \\
SPDD frontload & 5.0 & 0.054 & 0.078 & 0.125 & 0.203 & 77.7\% & 0.765 \\
\bottomrule
\end{tabular}
\end{table}

\begin{table}[htbp]
\centering
\caption{Slack mode ablation for playlist generation (reggaeton, per-position $\lambda$, $N{=}200$).}
\label{tab:playlist-slack}
\begin{tabular}{llrr}
\toprule
Slack mode & $\eta$ & Track (\%)$\uparrow$ & Div. \\
\midrule
Accumulated & 0.2 & 7.6 & 0.476 \\
Accumulated & 1.0 & 39.5 & 0.447 \\
Accumulated & 5.0 & 46.8 & 0.455 \\
\midrule
Instantaneous & 0.2 & 65.8 & 0.331 \\
Instantaneous & 1.0 & 80.7 & 0.210 \\
Instantaneous & 5.0 & 81.7 & 0.206 \\
\midrule
Early & 0.2 & 51.9 & 0.392 \\
Early & 1.0 & 75.6 & 0.256 \\
Early & 5.0 & 76.8 & 0.238 \\
\midrule
Optimistic (OMD) & 0.2 & 63.4 & 0.346 \\
Optimistic (OMD) & 1.0 & 79.8 & 0.217 \\
Optimistic (OMD) & 5.0 & 81.0 & 0.206 \\
\bottomrule
\end{tabular}
\end{table}

\begin{figure}[htbp]
\centering
\includegraphics[width=0.85\linewidth]{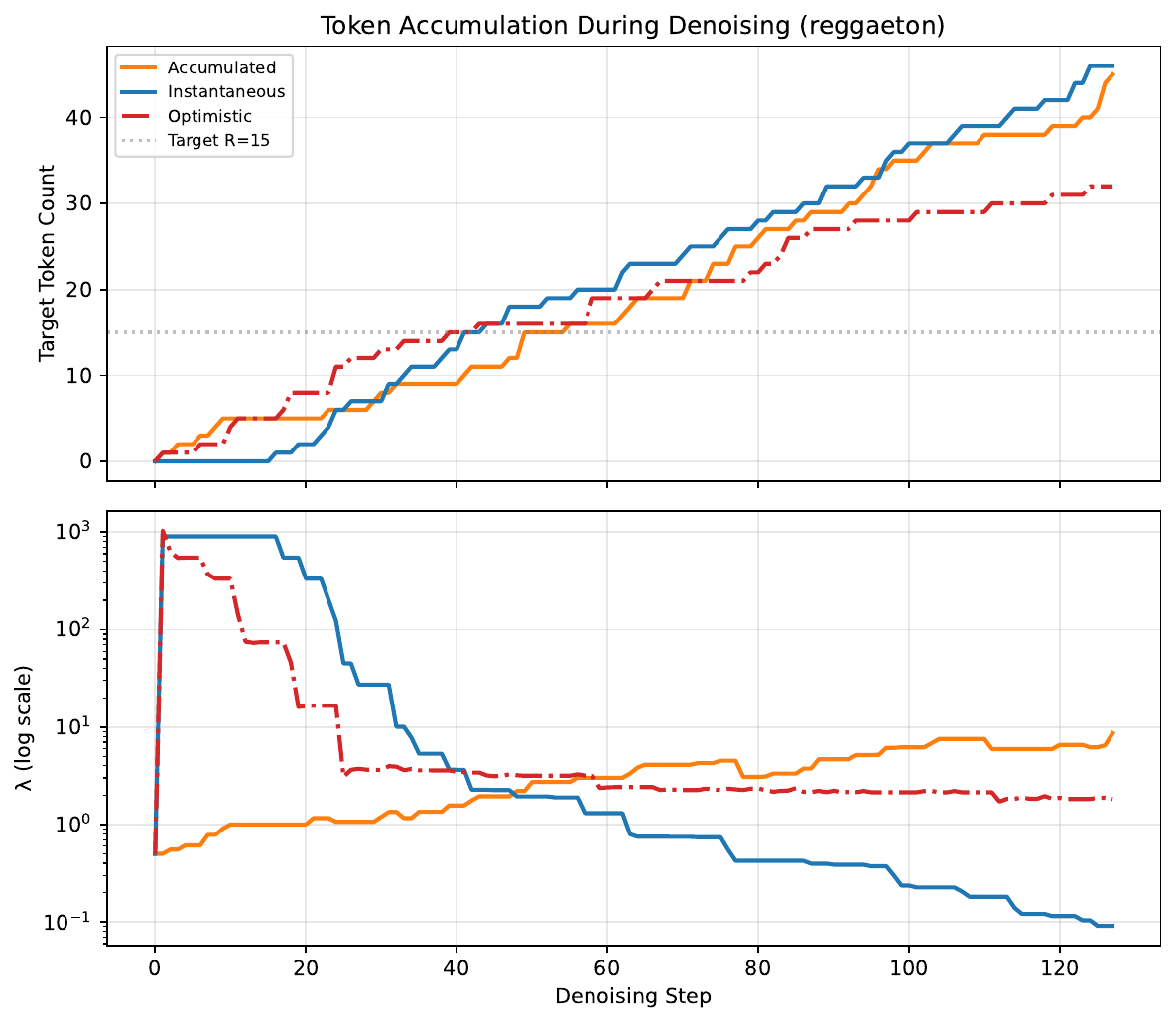}
\caption{Per-step dynamics during playlist generation (reggaeton, $\eta{=}0.5$, single sample). Top: target token accumulation. Bottom: $\lambda$ trajectory (log scale). Instantaneous slack drives $\lambda$ to the clamp ceiling immediately, front-loading all target tokens in the first ${\sim}20$ steps. Optimistic (OMD) slack rises gradually as the model's forecast $\hat{E}_k$ absorbs context.}
\label{fig:playlist-trace}
\end{figure}

\paragraph{Further observations.}

Embedding coherence follows a U-shape with increasing $\eta$: it dips at moderate guidance (0.65 at $\eta{=}0.5$) as the model explores genre-adjacent buckets, then recovers at strong guidance (0.73 at $\eta{=}5$) as it concentrates on high-scoring buckets; for bossa nova (9\% token coverage), diversity drops to 0.16 at $\eta{=}2$, reflecting the tight token pool. Per-position $\lambda$ outperforms scalar $\lambda$ by $+6$--$14\%$ relative (at $\eta{=}1.0$, 31.2\% vs.\ 27.3\%; at $\eta{=}5.0$, 34.1\% vs.\ 32.2\%); the advantage is smaller than on text because playlist sequences are shorter ($L{=}50$ vs.\ $L{=}256$). Artist diversity degrades moderately under strong guidance: unconstrained playlists average 43.6\% unique artists; SPDD ($\eta{=}5$) drops to 34.2\% for reggaeton and 30.5\% for thrash metal, reflecting the narrower candidate pool.

\newpage

\section*{NeurIPS Paper Checklist}

\begin{enumerate}

\item {\bf Claims}
    \item[] Question: Do the main claims made in the abstract and introduction accurately reflect the paper's contributions and scope?
    \item[] Answer: \answerYes{}
    \item[] Justification: The abstract claims inference-time constrained generation at $1\times$ cost with formal bounds and cross-domain generality. These are supported by Statement~\ref{th:cv} (theoretical bound), Algorithm~\ref{alg:pd} (method), and experiments across three domains (Section~5).

\item {\bf Limitations}
    \item[] Question: Does the paper discuss the limitations of the work performed by the authors?
    \item[] Answer: \answerYes{}
    \item[] Justification: The Limitations paragraph in the Discussion section addresses irrevocable token commitments in absorbing-state diffusion, noisy slack estimation for non-decomposable constraints, and the lack of hard satisfaction guarantees.

\item {\bf Theory assumptions and proofs}
    \item[] Question: For each theoretical result, does the paper provide the full set of assumptions and a complete (and correct) proof?
    \item[] Answer: \answerYes{}
    \item[] Justification: Statement~\ref{th:cv} provides a constraint violation bound under explicitly stated assumptions (Assumptions~\ref{as:markov} and~\ref{as:absorbing}). The full proof is in Appendix~\ref{app:proofs}.

    \item {\bf Experimental result reproducibility}
    \item[] Question: Does the paper fully disclose all the information needed to reproduce the main experimental results of the paper to the extent that it affects the main claims and/or conclusions of the paper (regardless of whether the code and data are provided or not)?
    \item[] Answer: \answerYes{}
    \item[] Justification: Model architectures, hyperparameters ($\eta$, $\lambda_0$, slack modes), training details, and dataset descriptions are provided in the experiment sections and appendix. Algorithm~\ref{alg:pd} fully specifies the method.

\item {\bf Open access to data and code}
    \item[] Question: Does the paper provide open access to the data and code, with sufficient instructions to faithfully reproduce the main experimental results, as described in supplemental material?
    \item[] Answer: \answerNo{}
    \item[] Justification: The text (TinyStories) and molecular (ZINC20) datasets are publicly available. The playlist dataset is proprietary.

\item {\bf Experimental setting/details}
    \item[] Question: Does the paper specify all the training and test details (e.g., data splits, hyperparameters, how they were chosen, type of optimizer) necessary to understand the results?
    \item[] Answer: \answerYes{}
    \item[] Justification: Training details (optimizer, learning rate, epochs, batch size, GPU type) are provided in each experiment section. Hyperparameter choices ($\eta$, $\lambda_0$) are reported in table captions and the appendix includes ablation studies.

\item {\bf Experiment statistical significance}
    \item[] Question: Does the paper report error bars suitably and correctly defined or other appropriate information about the statistical significance of the experiments?
    \item[] Answer: \answerYes{}
    \item[] Justification: Playlist and text experiments report standard errors ($\pm$ SE) computed across per-sample satisfaction. Statistical significance is reported via Mann-Whitney U tests ($p < 10^{-20}$) for playlist comparisons.

\item {\bf Experiments compute resources}
    \item[] Question: For each experiment, does the paper provide sufficient information on the computer resources (type of compute workers, memory, time of execution) needed to reproduce the experiments?
    \item[] Answer: \answerYes{}
    \item[] Justification: All experiments were run on RTX Pro 6000 GPUs with DDP. The method runs at $1\times$ cost (same as unconstrained sampling). Baseline costs ($2\times$ for D-CBG, $5$--$20\times$ for CDD) are reported.

\item {\bf Code of ethics}
    \item[] Question: Does the research conducted in the paper conform, in every respect, with the NeurIPS Code of Ethics \url{https://neurips.cc/public/EthicsGuidelines}?
    \item[] Answer: \answerYes{}
    \item[] Justification: The research uses publicly available datasets (TinyStories, ZINC20) and a proprietary playlist dataset with appropriate internal approvals. No human subjects are involved.

\item {\bf Broader impacts}
    \item[] Question: Does the paper discuss both potential positive societal impacts and negative societal impacts of the work performed?
    \item[] Answer: \answerNA{}
    \item[] Justification: This is a foundational method for constrained discrete generation. The primary applications (molecular design, text generation, playlist curation) do not raise direct negative societal concerns beyond those inherent in generative models generally.

\item {\bf Safeguards}
    \item[] Question: Does the paper describe safeguards that have been put in place for responsible release of data or models that have a high risk for misuse (e.g., pre-trained language models, image generators, or scraped datasets)?
    \item[] Answer: \answerNA{}
    \item[] Justification: The method is an inference-time decoding strategy, not a new pretrained model. The models used are small (21M--142M parameters) and domain-specific, posing no meaningful misuse risk.

\item {\bf Licenses for existing assets}
    \item[] Question: Are the creators or original owners of assets (e.g., code, data, models), used in the paper, properly credited and are the license and terms of use explicitly mentioned and properly respected?
    \item[] Answer: \answerYes{}
    \item[] Justification: TinyStories~\cite{eldan2023tinystories} and ZINC20~\cite{irwin2020zinc20} are cited with their original papers. Baseline methods (D-CBG~\cite{nisonoff2025unlocking}, CDD~\cite{cardei2025cdd}) are properly cited.

\item {\bf New assets}
    \item[] Question: Are new assets introduced in the paper well documented and is the documentation provided alongside the assets?
    \item[] Answer: \answerNA{}
    \item[] Justification: The paper does not release new datasets or pretrained models. The method is fully described in Algorithm~\ref{alg:pd}.

\item {\bf Crowdsourcing and research with human subjects}
    \item[] Question: For crowdsourcing experiments and research with human subjects, does the paper include the full text of instructions given to participants and screenshots, if applicable, as well as details about compensation (if any)?
    \item[] Answer: \answerNA{}
    \item[] Justification: No crowdsourcing or human subjects research was conducted.

\item {\bf Institutional review board (IRB) approvals or equivalent for research with human subjects}
    \item[] Question: Does the paper describe potential risks incurred by study participants, whether such risks were disclosed to the subjects, and whether Institutional Review Board (IRB) approvals (or an equivalent approval/review based on the requirements of your country or institution) were obtained?
    \item[] Answer: \answerNA{}
    \item[] Justification: No human subjects research was conducted.

\item {\bf Declaration of LLM usage}
    \item[] Question: Does the paper describe the usage of LLMs if it is an important, original, or non-standard component of the core methods in this research? Note that if the LLM is used only for writing, editing, or formatting purposes and does \emph{not} impact the core methodology, scientific rigor, or originality of the research, declaration is not required.
    \item[] Answer: \answerYes{}
    \item[] Justification: GPT-4.1-mini is used as a baseline for comparison in the text and molecular experiments (Tables~\ref{tab:ocean} and~\ref{tab:guidance}). It is not part of the proposed method.

\end{enumerate}

\end{document}